\documentclass{amsart}

\usepackage{amsmath,amssymb,amsthm}
\usepackage{algorithm}
\usepackage{algorithmic}
\usepackage{multirow}
\usepackage{color}
\usepackage{bm}
\usepackage{graphicx}
\usepackage{mathrsfs}
\usepackage{enumerate}
\usepackage{kbordermatrix}

\DeclareMathOperator*{\argmax}{argmax}
\DeclareMathOperator{\diag}{diag}

\DeclareMathOperator{\opt}{opt}
\DeclareMathOperator{\sign}{sign}
\DeclareMathOperator{\st}{s.t.}
\DeclareMathOperator{\tr}{tr}
\DeclareMathOperator{\T}{T}

\def\cR{{\mathcal R}}

\def\be{\pmb{e}}

\def\bp{\pmb{p}}
\def\bq{\pmb{q}}
\def\br{\pmb{r}}

\def\bu{\pmb{u}}

\def\bx{\pmb{x}}
\def\by{\pmb{y}}
\def\bz{\pmb{z}}

\def\bone{\pmb{1}}

\def\scrA{\mathscr{A}}

\def\wtd{\widetilde}
\def\what{\widehat}

\def\bbO{\mathbb{O}}
\def\bbR{\mathbb{R}}

\newtheorem{theorem}{Theorem}

\def\sss{\scriptscriptstyle}

\usepackage{lscape}

\title{Uncorrelated Semi-paired Subspace Learning}
\author[L. Wang] {Li~Wang}
\author[L. Zhang]{Lei-Hong~Zhang}
\author[C. Shen]{Chungen~Shen}
\author[R. Li]{Ren-Cang~Li,}
\thanks{
L. Wang is with Department of Mathematics and Department of Computer Science and Engineering, University of Texas at Arlington, Arlington, TX 76019-0408, USA. Email: li.wang@uta.edu. Corresponding Author.\\
Lei-Hong Zhang is with
School of Mathematical Sciences and Institute of Computational Science, Soochow University, Suzhou 215006, Jiangsu, China.\\
Chungen~Shen is with College of Science, University of Shanghai for Science and Technology, Shanghai 200093, China. Email: shenchungen@usst.edu.cn.\\
R. Li is with Department of Mathematics, University of Texas at Arlington, Arlington, TX 76019-0408, USA. Email: rcli@uta.edu.
}

\begin{document}
\maketitle

\begin{abstract}
Multi-view datasets are increasingly collected in many real-world applications, and we have seen 
better learning performance by existing multi-view learning methods than
by conventional single-view learning methods applied to each view individually. But, most of these multi-view learning
methods are built on the assumption that at each instance no view is missing and all data points from all views must be perfectly paired.
Hence they cannot handle unpaired data but ignore them completely from their learning process. However, unpaired data can be
more abundant in reality than paired ones and simply ignoring all unpaired data incur tremendous waste in resources.
In this paper,  we focus on learning uncorrelated features by semi-paired subspace learning, motivated
by many existing works that show great successes of learning uncorrelated features.
Specifically, we propose a generalized uncorrelated multi-view subspace learning framework, which can naturally integrate many
proven learning criteria on the semi-paired data. To showcase the flexibility of the framework,
we instantiate five new semi-paired models for both unsupervised  and semi-supervised learning.
We also design a successive alternating approximation (SAA) method to solve the resulting optimization problem
and the method can be combined with the powerful Krylov subspace projection technique if needed.
Extensive experimental results on multi-view feature extraction and multi-modality classification show that our proposed models perform competitively to or better than the baselines.
\end{abstract}

\section{Introduction}
In many real-world applications, datasets are increasingly collected for one underlying object in question from various aspects as real-world objects are often too complicated to be depicted by one aspect. Each aspect is referred to as a view. A dataset consisting of more than one aspects of an object is referred to as a multi-view dataset, otherwise a single-view dataset.
Multi-view datasets usually contain complementary, redundant, and corroborative characterizations of objects, and so are more informative than single-view datasets.
Although multi-view datasets are much more informative, learning from these datasets encounters tremendous challenges  \cite{baltruvsaitis2018multimodal}.

The most fundamental challenge is how multi-view data can be truthfully represented and summarized in such a way that
heterogeneity gaps \cite{peng2019cm} among different views can be satisfactorily overcome and comprehensive information concealed in multi-view data can be properly exploited by multi-view learning models.  A simple adaptation of existing single-view models cannot effectively handle  heterogeneity gaps because they do not take relationships among views into consideration.
As a consequence, a large number of multi-view learning methods have since been proposed
to narrow the heterogeneity gaps; see survey papers \cite{baltruvsaitis2018multimodal,li2018survey,zhao2017multi}
and references therein. Among them, multi-view subspace learning dominates as the most popularly
studied learning methodology, aiming to narrow or even eliminate the gaps.
As a representative multi-view subspace learning method, the
canonical correlation analysis (CCA) \cite{hardoon2004canonical} is widely used and has been adapted for various learning scenarios \cite{yang2019survey}.
The underlying foundational assumption in multi-view subspace learning is that
all views are generated from one common latent space via different transformations.

Another huge challenge to multi-view learning is that multi-view datasets from  real-world applications
are not often perfectly collected for all views. A complete multi-view dataset entails that data points at each instance must be collected.
In reality, that is hardly ever the case. In  other words real-world multi-view datasets are often incomplete in the sense
that some views may be missing at many instances.
In image/text classification, an image  may not always have its associated text description, and vice versa. In medical data analysis, patient may choose to skip some of the medical tests in diagnosis due to many reasons such as financial hardship, among others.
In the multi-view learning community, different terms were coined to call this type of multi-view datasets in the literature: semi-paired \cite{mehrkanoon2017regularized}, weakly-paired \cite{lampert2010weakly}, partial \cite{zhang2019cpm} and incomplete \cite{liu2020efficient}.
In this paper, we shall adopt the term \textit{semi-paired} to describe a multi-view dataset, a portion of which is paired while the rest unpaired.

Unfortunately,
most existing methods are not designed for semi-paired multi-view datasets. They cannot handle the unpaired portion of data but ignore them from their learning processes. That is a huge waste, given that the unpaired portion can be often much larger than the paired portion.
Semi-paired multi-view learning aims to take all data -- paired and unpaired -- into consideration for best learning performance in various learning scenarios. Two approaches are often used to handle  unpaired data. One approach is to fill in  missing views based on some criteria such as low-rank matrix completion \cite{zhang2018multi}, non-negative matrix completion \cite{hu2019doubly}, and probabilistic models \cite{zhang2014semi,kamada2015probabilistic}. However, estimating a large amount of missing views based on a small amount of paired data remains challenging.
Another approach is to fully explore the unpaired data through the common latent space under the framework of semi-paired subspace learning.
A number of learning criteria have been explored to capture the relationships between two views, such as cross covariance \cite{blaschko2008semi,kimura2013semicca,chen2012unified,mehrkanoon2017regularized} and between-view neighborhood graph \cite{zhou2013neighborhood,guo2018joint}. 
Cross covariance in the common space is used in CCA by maximizing the correlation between two views. It has since been extended to incorporate semi-paired data by either simultaneously maximizing the intra-view covariance of both paired and unpaired data \cite{kimura2013semicca} or minimizing the intra-view manifold regularization via kernel representation \cite{blaschko2008semi,mehrkanoon2017regularized}. Supervised information is also explored in \cite{chen2012unified} through maximizing the class separation of labeled data in the semi-supervised setting. The between-view neighborhood graph is another way to capture the inter-view relationships in analogy to the cross covariance in CCA. The between-view neighborhood graph is formed from two neighborhood graphs, each of which is constructed using all data of each view and the paired data as a bipartite graph  \cite{zhou2013neighborhood}, and it is later combined with two neighborhood graphs  \cite{guo2018joint} to simultaneously capture  intra-view relationships.

In this paper, we are particularly interested in semi-paired subspace learning with uncorrelated
feature constraints in the common space. For feature extraction, it is expected that all extracted features should be mutually uncorrelated \cite{jin2001face}. This is inspired from the observation that the accuracy of statistical classifiers increases as the number of features increases up to a tipping point and then decreases from that point on \cite{hughes1968mean},
in part because as the number of extracted features becomes too large, redundancy creeps in among the extracted features.
That particular tipping point corresponds to the optimal number of extracted features, representing
the minimal set of optimal features for best classification performance.
Requiring extracted features to be uncorrelated provides a way to ensure minimal or no redundancy among extracted features.

Subspace learning methods enforcing uncorrelated features have successfully been explored for single-view learning \cite{nie2009semi,zhan:2011,jin2001face,ye2004using} and multi-view learning  \cite{sun2015multiview}. These methods have demonstrated great successes in many applications, but they are not capable of learning uncorrelated features for data with missing views, and  are further restricted to either supervised learning or fully paired data analysis.

In this paper, we will develop a generalized semi-paired subspace learning framework to learn uncorrelated features in the latent common space. The framework can naturally integrate different types of learning criteria, such as covariance, class separability and manifold regularization, into it. As showcases, we demonstrate the capability of our framework by deriving novel semi-paired models for both unsupervised and semi-supervised settings. To solve the resulting challenging optimization problems, we also propose an efficient algorithm.

\noindent\textbf{Contributions.} The main contributions of this paper are summarized as follows:
\begin{itemize}
	\item A generalized uncorrelated multi-view subspace learning framework is proposed. The framework can naturally integrate relationships between two views, supervised information, unpaired data, and can simultaneously learn uncorrelated features.
	
	\item It is a versatile framework that can be adapted to solve various semi-paired learning problems.
	To demonstrate its flexibility, new models are instantiated for unsupervised semi-paired learning and semi-supervised semi-paired learning.
	
	\item The framework is stated in the form of a challenging optimization problem with projection matrices as variables.
	A successive alternating approximation (SAA) method is proposed to find approximations of the optimizer. The method
	when combined with Krylov subspace projection techniques is suitable for practical purposes.
		
	\item Extensive experiments are conducted for evaluating the proposed models against existing methods in terms of multi-view feature extraction and multi-modality classification. Experimental results show that our proposed models perform competitively to or better than baselines.
\end{itemize}

\noindent\textbf{Paper organization.} We first explain three equivalent formulations of CCA and review existing semi-paired subspace learning methods in Section~\ref{sec:related-work}. In Section~\ref{sec:spocca}, we propose the generalized uncorrelated semi-paired subspace learning framework, and new models instantiated from it. The proposed optimization algorithm is presented in Section~\ref{sec:algo}. Extensive experiments are conducted in Section~\ref{sec:experiments}. Finally, we draw our conclusions in Section~\ref{sec:conclusions}.

\noindent\textbf{Notation.}
$\bbR^{m\times n}$ is the set of $m\times n$ real matrices and $\bbR^n=\bbR^{n\times 1}$. $I_n\in\bbR^{n\times n}$ is
the identity matrix, $\be_j$ is its $j$th column (whose dimension can be inferred from the context), and $\bone_n\in\bbR^n$ is the vector of all ones. $\|\bx\|_2$
is the 2-norm of a vector $\bx\in \bbR^n$. For $B\in\bbR^{m\times n}$,
$\cR(B)$ is its column subspace.
The Stiefel manifold
$$
\bbO^{n\times k}=\{X\in\bbR^{n\times k}\,:\, X^{\T}X=I_k\}
$$
Notation $B_{(i,j)}$ and $\bx_{(i)}$ take out entries of a matrix and vector, respectively.

\section{Related Work} \label{sec:related-work}

\subsection{CCA: three equivalent formulations} \label{sec:cca}
Denote the paired two-view dataset by $\{ (\bx_i^{(1)}, \bx_i^{(2)}) \}_{i=1}^n$, where $\bx_i^{(s)} \in \bbR^{d_s}$ is the $i$th data points of view $s$ for $s=1, 2$, and $n$ is the number of paired data points. Let
\begin{align}
X_s = \begin{bmatrix}
\bx_1^{(s)},
\bx_2^{(s)},
\ldots,
\bx_n^{(s)}
\end{bmatrix} \in \bbR^{d_s \times n}, \forall s=1,2. \label{eq:Xp}
\end{align}
The classical CCA aims to learn
projection matrices
$P_s \in \bbR^{d_s \times k}$
so that the correlation between two views in the common space $\bbR^k$ are maximized under the constraints that they are uncorrelated and of unit variance \cite{hardoon2004canonical}. Specifically,
for $\bx_i^{(s)} \in \bbR^{d_s}$ of view $s$, its projected point $\bz_i^{(s)} \in \bbR^k$  is given by, for $i=1,\ldots, n$,
\begin{align}
\bz_i^{(s)} = P_s^{\T} \bx_i^{(s)}, \forall s=1,2.
\end{align}
Accordingly, they are collectively  represented as
\begin{align}
Z_s = [\bz_1^{(s)}, \ldots, \bz_n^{(s)} ] = P_s^{\T} X_s \in \bbR^{k \times n}, \forall s=1,2.
\end{align}
Denote the (cross-)covariance matrix of view $s$ and view $t$ by
\begin{align}
C_{s,t}=\frac{1}{n} X_s H_n X_t^{\T} \in \bbR^{d_s \times d_t}, \forall s, t=1,2,
\end{align}
where $H_n=I_n - \frac{1}{n}\bone_n \bone_n^{\T}$ is the centering matrix.
Define
$$
\overline{\bz}_s =  \frac{1}{n}  \sum_{i=1}^n \bz_i^{(s)} =  \frac{1}{n}  Z_s \bone_n,
$$
the mean of view $s$. We have the sample cross-covariance matrix between $Z_1$ and $Z_2$ in the common space given by
\begin{align}
P_1^{\T} C_{1,2} P_2 &= \frac{1}{n}Z_1 H_n Z_2^{\T} \nonumber\\
&= \frac{1}{n} \sum_{i=1}^n\big( \bz_i^{(1)} -  \overline{\bz}_1 \big) \big( \bz_i^{(2)} -  \overline{\bz}_2 \big)^{\T}.
\label{eq:PCP2ZZ}
\end{align}
Similarly, the sample covariance matrix of $Z_s$ is $P_s^{\T} C_{s,s} P_s$. In what follows, 
we discuss three equivalent formulations of CCA, each of which  affords a different interpretation.

\subsubsection{Fractional formulation} \label{sec:fractional}
By definition, CCA aims to maximize the correlation between two views which is formulated as a fractional maximization problem
\begin{subequations} \label{eq:cca}
	\begin{alignat}{2}
	\max_{P_1, P_2} & \frac{ \tr(P_1^{\T} C_{1,2} P_2) }{\sqrt{\tr(P_1^{\T} C_{1,1} P_1) } \sqrt{\tr( P_2^{\T} C_{2,2} P_2)}} \label{op:cca-fractional}\\
	\st &~ P_1^{\T} C_{1,1} P_1 = P_2^{\T} C_{2,2} P_2 = I_k, \label{con:unit-variance}
	\end{alignat}
\end{subequations}%
where 
constraints \eqref{con:unit-variance} impose uncorrelation and unit-variance on $Z_s$ because of
\eqref{eq:PCP2ZZ}. Optimal solution pair $(P_1,P_2)$ of CCA \eqref{eq:cca} admits
the invariant property that if $(P_1,P_2)$ is an optimal solution pair then so is
$(P_1Q_1,P_2Q_2)$ for any orthogonal matrices $Q_s\in\bbR^{k\times k},~s = 1,2$. In particular, we often take
the one among them such that
\begin{equation}\label{eq:PCP=diag}
P_1^{\T} C_{1,2} P_2\quad\mbox{is diagonal},
\end{equation}
and when it happens, necessarily the diagonal entries are nonnegative.
With that, the corresponding
$Z_1$ and $Z_2$ are inter-uncorrelated

Problem (\ref{eq:cca}) can be reformulated as a singular value decomposition (SVD) problem.
Suppose that $C_{1,1}$ and $C_{2,2}$ are positive definite and let
\begin{align}
\what{P}_1 = C_{1,1}^{1/2} P_1,  \what{P}_2= C_{2,2}^{1/2} P_2.
\end{align}
Then we can rewrite (\ref{eq:cca})  as
\begin{subequations} \label{eq:cca2}
	\begin{align}
	\max_{\what{P}_1, \what{P}_2}  & \tr(\what{P}_1^{\T} C_{1,1}^{-1/2}  C_{1,2} C_{2,2}^{-1/2} \what{P}_2) \label{op:svd-form}\\
	\st &~ \what{P}_1^{\T} \what{P}_1 = \what{P}_2^{\T} \what{P}_2 = I_k.
	\end{align}
\end{subequations}
Let the SVD of $C_{1,1}^{-1/2}  C_{1,2} C_{2,2}^{-1/2}$ be \cite{govl:2013}
\begin{equation}\label{eq:CCC-SVD}
C_{1,1}^{-1/2}  C_{1,2} C_{2,2}^{-1/2} = U \Gamma V^{\T},
\end{equation}
where $U\in\bbR^{d_1\times d_1}$ and $V\in\bbR^{d_2\times d_2}$ are
orthogonal, and $\Gamma$ is diagonal with diagonal entries
being the singular values, arranged in the descending order.
It is well-known that the optimal objective value of \eqref{eq:cca2} is the sum of the top $k$ singular
values \cite{hojo:2013}.
We have
\begin{align}
U_{(:,1:k)}^{\T} U_{(:,1:k)} &= V_{(:,1:k)}^{\T} V_{(:,1:k)} = I_k, \label{eq:svd-1}\\
U_{(:,1:k)}^{\T} C_{1,1}^{-1/2}  C_{1,2} C_{2,2}^{-1/2} V_{(:,1:k)} &=  U_{(:,1:k)}^{\T} U \Gamma V^{\T} V_{(:,1:k)} \nonumber \\
& = \Gamma_{(1:k,1:k)}, \label{eq:svd-2}
\end{align}
which means that $(\what{P}_1^{\opt},\what{P}_2^{\opt}) = (U_{(:,1:k)}, V_{(:,1:k)})$
is feasible and is an optimal solution pair because $\tr(\Gamma_{(1:k,1:k)})$ is
the sum of the top $k$ singular
values of $C_{1,1}^{-1/2}  C_{1,2} C_{2,2}^{-1/2}$. 
Accordingly, an  optimal solution pair $(P_1^{\opt},P_2^{\opt})$ of \eqref{eq:cca} can be recovered by
\begin{align}
P_1^{\opt} = C_{1,1}^{-1/2} U_{(:,1:k)},\,\, P_2^{\opt} = C_{2,2}^{-1/2} V_{(:,1:k)}. \label{eq:svd-transform}
\end{align}
It can be verified that \eqref{eq:PCP=diag} also holds.

\subsubsection{Generalized eigenvalue formulation}
The SVD \eqref{eq:CCC-SVD} leads to
\begin{subequations}\label{eq:SVD-imply}
	\begin{align}
	C_{1,1}^{-1/2}  C_{1,2} C_{2,2}^{-1/2} V_{(:,1:k)} = U_{(:,1:k)} \Gamma_{(1:k,1:k)}, \label{eq:SVD-imply-1}\\
	\big[C_{1,1}^{-1/2}  C_{1,2} C_{2,2}^{-1/2} \big]^{\T} U_{(:,1:k)} = V_{(:,1:k)} \Gamma_{(1:k,1:k)}. \label{eq:SVD-imply-2}
	\end{align}
\end{subequations}
Putting them together yields
\begin{multline}\label{eq:svd2eig}
\begin{bmatrix}
0 & C_{1,1}^{-1/2}  C_{1,2} C_{2,2}^{-1/2} \\
(C_{1,1}^{-1/2}  C_{1,2} C_{2,2}^{-1/2} )^{\T} & 0
\end{bmatrix} \begin{bmatrix}
U_{(:,1:k)} \\
V_{(:,1:k)}
\end{bmatrix} 
=\begin{bmatrix}
U_{(:,1:k)} \\
V_{(:,1:k)}
\end{bmatrix}   \Gamma_{(1:k,1:k)}.
\end{multline} 
Since
\begin{multline*}
\begin{bmatrix}
0 & C_{1,1}^{-1/2}  C_{1,2} C_{2,2}^{-1/2} \\
\big[C_{1,1}^{-1/2}  C_{1,2} C_{2,2}^{-1/2}\big]^{\T} & 0
\end{bmatrix} \\
= \begin{bmatrix}
C_{1,1}^{-1/2}  & 0\\
0 & C_{2,2}^{-1/2}
\end{bmatrix}\begin{bmatrix}
0 &  C_{1,2} \\
C_{1,2}^{\T} & 0
\end{bmatrix}\begin{bmatrix}
C_{1,1}^{-1/2}  & 0\\
0 & C_{2,2}^{-1/2}
\end{bmatrix},
\end{multline*} 
using (\ref{eq:svd-transform})
we can rewrite (\ref{eq:svd2eig}) into the form of a generalized eigenvalue problem (GEP):
\begin{align}
& \begin{bmatrix}
0& C_{1,2}\\
C_{1,2}^{\T} & 0
\end{bmatrix} \begin{bmatrix}
P_1^{\opt} \\
P_2^{\opt}
\end{bmatrix}  \nonumber\\
&= \begin{bmatrix}
0& C_{1,2}\\
C_{1,2}^{\T} & 0
\end{bmatrix} \begin{bmatrix}
C_{1,1}^{-1/2}  & 0\\
0 & C_{2,2}^{-1/2}
\end{bmatrix} \begin{bmatrix}
U_{(:,1:k)} \\
V_{(:,1:k)}
\end{bmatrix} \nonumber\\
&= \begin{bmatrix}
C_{1,1} & 0\\
0 & C_{2,2}
\end{bmatrix} \begin{bmatrix}
C_{1,1}^{-1/2}  & 0\\
0 & C_{2,2}^{-1/2}
\end{bmatrix} \begin{bmatrix}
U_{(:,1:k)} \\
V_{(:,1:k)}
\end{bmatrix} \Gamma_{(1:k,1:k)} \nonumber\\
&=  \begin{bmatrix}
C_{1,1} & 0\\
0 & C_{2,2}
\end{bmatrix}\begin{bmatrix}
P_1^{\opt} \\
P_2^{\opt}
\end{bmatrix} \Gamma_{(1:k,1:k)},  \label{eq:gev}
\end{align}
where the first and third equalities hold because of (\ref{eq:svd-transform}) and the second
equality is due to \eqref{eq:SVD-imply}.
Hence as a corollary of Ky Fan's maximum principle \cite[p.35]{bhat:1996},
$(P_1^{\opt}/2,P_2^{\opt}/2)$ as determined by (\ref{eq:svd-transform})
is an optimal solution pair of
\begin{subequations} \label{eq:cca3}
	\begin{alignat}{2}
	\max_{P_1, P_2} & \tr\left( [P_1^{\T}, P_2^{\T}]  \begin{bmatrix}
	0& C_{1,2}\\
	C_{1,2}^{\T} & 0
	\end{bmatrix}  \begin{bmatrix}
	P_1 \\
	P_2
	\end{bmatrix}  \right) \\
	\st &~  [P_1^{\T}, P_2^{\T}] \begin{bmatrix}
	C_{1,1} & 0\\
	0 & C_{2,2}
	\end{bmatrix}\begin{bmatrix}
	P_1 \\
	P_2
	\end{bmatrix}  = I_k.
	\end{alignat}
\end{subequations}
Or, equivalently,
\begin{subequations} \label{eq:cca4}
	\begin{alignat}{2}
	\max_{P_1, P_2} & \tr(P_1^{\T} C_{1,2} P_2)  \label{op:gev} \\
	\st &~ P_1^{\T} C_{1,1} P_1 + P_2^{\T} C_{2,2} P_2 = I_k. \label{op:con:gev}
	\end{alignat}
\end{subequations}

\subsubsection{Uncorrelated constrained optimization} \label{sec:ucca}
Under the constraints in \eqref{con:unit-variance}, the denominator in the objective function in \eqref{op:cca-fractional}
is constant,
and thus \eqref{eq:cca} becomes
\begin{subequations} \label{eq:cca5}
	\begin{alignat}{2}
	\max_{P_1, P_2} &  \tr(P_1^{\T} C_{1,2} P_2) \label{op:cca5-1}\\
	\st &~ P_1^{\T} C_{1,1} P_1 = P_2^{\T} C_{2,2} P_2 = I_k. \label{op:cca5-2}
	\end{alignat}
\end{subequations}
This is similar to \eqref{eq:cca4}, except that it imposes  uncorrelation and unit variance on projected points in each view
as \eqref{op:cca5-2}, rather than the correlated constraints across views as \eqref{op:con:gev}.

\subsection{Semi-paired CCA} \label{sec:sp-cca}
Semi-paired CCA is generally formulated as  extensions of CCA in the form of GEP to incorporate unpaired data from two views \cite{mehrkanoon2017regularized,blaschko2008semi,kimura2013semicca,chen2012unified} in order to improve learning performance.

In the classical CCA, data points from each view are well paired, i.e., each $\bx_i^{(1)}$ of view 1 pairs
with $\bx_i^{(2)}$ of view 2 for all instances $1\le i\le n$. In practice, not every collected data point $\bx^{(1)}$ from view 1 has a corresponding
collected data point $\bx^{(2)}$ from view 2 to go with it. Hence, we may have portion of  data  paired while the rest, often
a larger portion, of data unpaired. Such data are said to be semi-paired.

Suppose, in addition to the paired data $X_1$ and $X_2$ in (\ref{eq:Xp}), we have unpaired data  $\{\bx_i^{(1)}\}_{i=n+1}^{n_1}$  from view $1$, and $\{\bx_i^{(2)}\}_{i=n+1}^{n_2}$ from view $2$. I.e., we have $n_1-n$ unpaired data points for view 1 and $n_2-n$ unpaired data points for view $2$, respectively. Let
\begin{align}
\widetilde{X}_s = \begin{bmatrix}
\bx_1^{(s)},
\bx_2^{(s)},
\ldots,
\bx_{n_s}^{(s)}
\end{bmatrix} \in \bbR^{d_s \times n_s}, \forall s=1,2,
\end{align}
of all collected data, paired and unpaired, for view $s$.
Moreover, we suppose that a small amount of data are labeled for each view.
The labeled data can be paired or unpaired.

Below, we will briefly review two representative semi-paired CCA models in both
unsupervised and semi-supervised settings.

\subsubsection{Unsupervised learning} \label{sec:us-spcca}
Define the total covariance for each view:
\begin{align}
\widetilde{C}_{s,s} = \frac{1}{n_s}\widetilde{X}_s H_{n_s} \widetilde{X}_s^{\T}, \quad \forall s=1,2.
\end{align}
Analogously to \eqref{eq:gev},
SemiCCA \cite{kimura2013semicca} incorporates unpaired data by solving
\begin{multline}\label{eq:semicca}
\begin{bmatrix}
(1-\gamma) \widetilde{C}_{1,1} & \gamma C_{1, 2} \\
\gamma C_{1,2}^{\T} & (1-\gamma) \widetilde{C}_{2,2}
\end{bmatrix}
\begin{bmatrix}
P_1 \\
P_2
\end{bmatrix} \\
\!\!\!\!\!\!=
\begin{bmatrix}
\gamma C_{1,1} + (1-\gamma)  I_{d_1} & 0\\ 
0&\gamma C_{2,2} + (1-\gamma)  I_{d_2}
\end{bmatrix} \begin{bmatrix}
P_1 \\
P_2
\end{bmatrix} \Lambda,
\end{multline}
 
where parameter $\gamma \in [0,1]$ controls the tradeoff between CCA on
the paired data and PCA on all data. If $\gamma=1$, (\ref{eq:semicca}) reduces to CCA on the paired data only.
However, if $\gamma=0$, it does not reduce to two separate PCA on all data for each view because of the shared
$\Lambda$ that picks up the top $k$ eigenvalues out of those of both
$\widetilde{C}_{1,1}$ and $\widetilde{C}_{2,2}$.

The primal problem of (\ref{eq:semicca}) is given by the following optimization problem \cite{chen2012unified}
\begin{subequations}\label{op:semicca-primal}
	\begin{align}
	\max_{P_1, P_2} &~ \gamma \tr(P_1^{\T} C_{1,2} P_2) + \frac{1-\gamma}{2} \sum_{s=1}^2 \tr(P_s^{\T} \widetilde{C}_{s,s} P_s) \label{op:semicca-primal-1}\\
	\st &~ \sum_{s=1}^2 [\gamma P_s^{\T} C_{s,s} P_s + (1-\gamma ) P_s^{\T}P_s]= I_k. \label{op:semicca-primal-2}
	\end{align}
\end{subequations}

SemiCCA with Laplacian regularization (SemiCCALR) \cite{blaschko2008semi} incorporates unpaired data into CCA
in the form of GEP
\begin{align}
\!\!\!\!\begin{bmatrix}
0 & C_{1, 2} \\
C_{1,2}^{\T} & 0
\end{bmatrix} \begin{bmatrix}
P_1 \\
P_2
\end{bmatrix}
\!\!=\!\!
\begin{bmatrix}
C_{1,1} + R_1& 0 \\
0&C_{2,2} + R_2
\end{bmatrix} \begin{bmatrix}
P_1 \\
P_2
\end{bmatrix} \Lambda, \label{eq:semicca-lr}
\end{align}
where $R_s = \gamma_1 I_{d_s} + \gamma_2 \widetilde{X}_s L_s \widetilde{X}_s^{\T}$ with  graph Laplacian matrices $L_s=\textrm{diag}(W_s \bone_{n_s}) - W_s$ of graph $W_s \in \bbR^{n_s \times n_s}$ of view $s$ for $s=1,2$, and $\gamma_1, \gamma_2$ are regularization parameters. Its primal problem is
\begin{equation}\label{op:semicca-lr}
\max_{P_1, P_2} \, \tr(P_1^{\T} C_{1,2} P_2)\,\,
\st ~ \sum_{s=1}^2 P_s^{\T} (C_{s,s} + R_s) P_s = I_k,
\end{equation}
which differs from CCA in the form (\ref{eq:cca4}) only in replacing $C_{s,s}$ there by $C_{s,s} + R_s$ which involves all data -- paired and unpaired.

\subsubsection{Semi-supervised learning} \label{sec:ss-spcca}
In \cite{chen2012unified}, the supervised class labels are incorporated into SemiCCA for multi-class classification.
Let $\{ (\bx_{s_i}^{(s)}, y_i^{(s)}) \}_{i=1}^{m_s}$ be the $m_s$  data points from view $s$ whose labels are known,
where $s_i \in \{1,\ldots, n_s\}$ for $1\le i\le m_s$ are the indices of label data points of view $s$ and $y_i^{(s)} \in \{1, \ldots, c\}$ of $c$ classes are the corresponding labels. The labeled data points can come from both paired and unpaired portions of the data.
Pack the labeled data points to get
\begin{align}
\widehat{X}_s = \begin{bmatrix}
\bx_{s_1}^{(s)},
\bx_{s_2}^{(s)},
\ldots,
\bx_{s_{m_s}}^{(s)}
\end{bmatrix} \in \bbR^{d_s \times m_s}, \forall s=1,2,
\end{align}
and let $\widehat{Y}_s \in \{0,1\}^{c \times m_s}$ be its corresponding label matrix  obtained via one-hot representation:
\begin{align}
\big[\widehat{Y}_s\big]_{(r,i)} =
\begin{cases}
1, & \mbox{if $y_i^{(s)}= r$}, \\
0, & \textrm{otherwise,}
\end{cases}
\end{align}
for $r=1,\ldots,c$ and $i=1,\ldots,m_s$.
In LDA, the within-class scatter matrix $S_{\rm w}^{(s)}$ and between-class scatter matrix $S_{\rm b}^{(s)}$ for  view $s$ are defined as
\begin{subequations}\label{eq:S4LDA}
	\begin{align}
	&S_{\rm w}^{(s)} \!\!=\!\!\frac{1}{m_s}\widehat{X}_s \widehat{L}_{\rm w}^{(s)} \widehat{X}_s^{\T},\,\,
	\widehat{L}_{\rm w}^{(s)}\!\!=\!\! \textrm{diag}(\widehat{W}_{\rm w}^{(s)} \bone_{m_s}) - \widehat{W}_{\rm w}^{(s)}, \label{eq:S4LDA-2}\\
	&S_{\rm b}^{(s)} \!\!=\!\! \frac{1}{m_s} \widehat{X}_s \widehat{L}_{\rm b}^{(s)} \widehat{X}_s^{\T},\,\,
	\widehat{L}_{\rm b}^{(s)}\!\!=\!\! \textrm{diag}(\widehat{W}_{\rm b}^{(s)} \bone_{m_s}) - \widehat{W}_{\rm b}^{(s)}, \label{eq:S4LDA-b}
	\end{align}
\end{subequations}
where  graph matrices $\widehat{W}_{\rm w}^{(s)}$ and $\widehat{W}_{\rm b}^{(s)}$  are defined as
\begin{align}
\widehat{W}_{\rm w}^{(s)} \!\!=\!\! \widehat{Y}_s^{\T} (\widehat{Y}_s \widehat{Y}_s^{\T})^{-1} \widehat{Y}_s,\,\,
\widehat{W}_{\rm b}^{(s)} \!\!=\!\! \frac{1}{m_s} \bone_{m_s}  \bone_{m_s}^{\T} - \widehat{W}_{\rm w}^{(s)}. \label{eq:W-lda}
\end{align}
Note that $\widehat{Y}_s \widehat{Y}_s^{\T}$ is a diagonal matrix whose $(r, r)$th entry, denoted by $m_r^{(s)}$,
is the number of data points of view $s$ in class $r$:
$$
m_r^{(s)} = \sum_{i=1}^{m_s} \big[\widehat{Y}_s\big]_{(r,i)}.
$$ 
There are other approaches for constructing the above scatter matrices for different situations, too, such as local Fisher discriminant analysis (LFDA) \cite{chen2012unified} and marginal Fisher analysis (MFA) \cite{yan2006graph}.

Finally,  
S$^2$GCA \cite{chen2012unified} based on LFDA is formulated
as solving
\begin{multline}\label{eq:s2gca}
\begin{bmatrix}
M_{1,1} &  \gamma C_{1, 2} \\
\gamma C_{1,2}^{\T} & M_{2,2}
\end{bmatrix} \begin{bmatrix}
P_1 \\
P_2
\end{bmatrix}  \\
\!=\!
\begin{bmatrix}
\gamma C_{1,1} + (1 - \gamma) I_{d_1} & 0 \\
0&\gamma C_{2,2} +  (1-\gamma) I_{d_2}
\end{bmatrix} \begin{bmatrix}
P_1 \\
P_2
\end{bmatrix} \Lambda,
\end{multline} 
where
$M_{s,s} = \eta (S_{\rm b}^{(s)} - S_{\rm w}^{(s)}) + (1-\gamma) \widetilde{C}_{s,s}, ~ \forall s=1,2$, $\gamma \in [0,1]$ and $\eta$ are trade-off parameters. The primal problem of S$^2$GCA \eqref{eq:s2gca} is
\begin{subequations}\label{op:s2gca-primal}
	\begin{align}
	\max_{P_1, P_2} &~ \gamma \tr(P_1^{\T} C_{1,2} P_2) + \frac{1}{2} \sum_{s=1}^2 \tr(P_s^{\T} M_{s,s} P_s) \label{op:s2gca-primal-1}\\
	\st &~ \sum_{s=1}^2 [\gamma P_s^{\T} C_{s,s} P_s + (1-\gamma ) P_s^{\T}P_s]= I_k. \label{op:s2gca-primal-2}
	\end{align}
\end{subequations}

It is worth noting that when $\eta=0$, (\ref{eq:s2gca}) reduces to SemiCCA (\ref{eq:semicca}) in the unsupervised setting since $M_{s,s} = (1-\gamma) \widetilde{C}_{s,s}$. For $\eta > 0$, S$^2$GCA can be interpreted as a modified SemiCCA by incorporating supervised information through LFDA. Similarly,  the graph similarity matrices (\ref{eq:W-lda}) based on LDA and MFA can be employed. 

The approach of adding supervised information into SemiCCA can also be used to improve CCA and SemiCCALR:
\begin{subequations}\label{op:scca}
	\begin{align}
	\max_{P_1, P_2} &  \tr(P_1^{\T} C_{1,2} P_2) + \frac{\eta}{2} \sum_{s=1}^2 \tr(P_s^{\T} (S_{\rm b}^{(s)} - S_{\rm w}^{(s)})   P_s)\label{op:scca-1}\\
	\st &~  P_1^{\T} C_{1,1}^{(1)} P_1 +  P_2^{\T} C_{2,2} P_2 = I_k, \label{op:scca-2}
	\end{align}
\end{subequations}
and
\begin{subequations}\label{op:sscca-lr}
	\begin{align}
	\max_{P_1, P_2} &~ \tr(P_1^{\T} C_{1,2} P_2) + \frac{\eta}{2} \sum_{s=1}^2 \tr(P_s^{\T} (S_{\rm b}^{(s)} - S_{\rm w}^{(s)})  P_s)
	\label{op:sscca-lr-1} \\
	\st &~  \sum_{s=1}^2 P_s^{\T} (C_{s,s}+ R_s) P_s = I_k, \label{op:sscca-lr-2}
	\end{align}
\end{subequations}
named as SCCA and S$^2$CCALR respectively, to be used as two of the baseline methods in
Section~\ref{sec:experiments}.

\section{Uncorrelated Semi-paired Learning} \label{sec:spocca}
In Section~\ref{sec:sp-cca}, three representative semi-paired subspace learning models are reviewed. They turn into
GEP. Inspired by the three equivalent formulations of CCA in Section~\ref{sec:cca}, we will propose a novel semi-paired multi-view subspace learning framework with unit regularized covariance of projected data points in each view separately. 

In what follows, we first introduce the proposed framework in Section~\ref{ssec:USLF}, and then showcase several new models
in Section~\ref{sec:our-models}.

\subsection{Uncorrelated semi-paired learning framework}\label{ssec:USLF}
``Uncorrelated'' learning is about ensuring orthogonality among features of projected data points in the reduced common space. Specifically, it is to enforce, e.g., the covariance matrices of projected data points, denoted by $Z_sH_nZ_s^{\T} = P_s^{\T} C_{s,s} P_s\,\, \forall s$, are diagonal. The notion has been employed in both single-view \cite{jin2001face,ye2006feature,nie2009semi,zhang2011uncorrelated} and multi-view subspace learning models \cite{sun2015multiview,yin2019multiview,shu2019multi}.
It has been observed
that uncorrelated features learned  by these models can generally outperform correlated features. In this paper, we regard any constraint like $Z_s A Z_s^{\T}$ being diagonal as ``uncorrelated'' features where $A$ is a positive semi-definite matrix.

Imposing the uncorrelated property has mostly been explored for supervised models or fully paired data, but has not been yet explored for semi-paired subspace learning.
CCA formulated as
(\ref{eq:cca5}) naturally fits into this notion of uncorrelated learning.
However, all existing semi-paired CCA models are built on formulation (\ref{eq:cca4}) that is opposite to uncorrelated learning.
Abstracting from the semi-paired models in
Section~\ref{sec:sp-cca}, we propose the following uncorrelated semi-paired subspace learning framework
\begin{subequations} \label{op:framework0}
	\begin{alignat}{2}
	\max_{P_1, P_2} & \tr(P_1^{\T} \Phi_{1,2} P_2) + \frac{1}{2}\sum_{s=1}^2 \tr(P_s^{\T} \Phi_{s,s} P_s) \label{op:framework}\\
	\st &~ P_1^{\T} \Psi_{1,1} P_1 = P_2^{\T} \Psi_{2,2} P_2 = I_k, \label{con:framework}
	\end{alignat}
\end{subequations}
where $\Phi_{s,t} \in \bbR^{d_s \times d_t}$ and $\Psi_{s,s}\in \bbR^{d_s \times d_s}, \forall s, t=1,2$ are matrices 
to be defined, as in the corresponding models in Section~\ref{sec:our-models} below.

In general, the proposed framework (\ref{op:framework0}) does not
admit similar equivalent formulations to those of CCA in Section~\ref{sec:cca}.
Consider the Lagrangian function of (\ref{op:framework0})
\begin{align*}
L(P_1, P_2) &= \tr(P_1^{\T} \Phi_{1,2} P_2) + \frac{1}{2}\sum_{s=1}^2 \tr(P_s^{\T} \Phi_{s,s} P_s)   - \frac 12\sum_{s=1}^2 \tr((P_s^{\T} \Psi_{s,s} P_s - I_k) \Lambda_s),
\end{align*}
where Lagrangian multipliers $\Lambda_1 \in \bbR^{k \times k}$ and $\Lambda_2 \in \bbR^{k \times k}$
are symmetric. Hence,
the KKT conditions of \eqref{op:framework0} are
\begin{align*}
\Phi_{1,2} P_2 + \Phi_{1,1} P_1 &= \Psi_{1,1} P_1 \Lambda_1, \\
\Phi_{1,2}^{\T} P_1 + \Phi_{2,2} P_2 &= \Psi_{2,2} P_2 \Lambda_2, \\
P_s^{\T} \Psi_{s,s} P_s &= I_k\,\, \forall s=1,2.
\end{align*}
Rearrange these equations to give 
\begin{align}
\begin{bmatrix}
\Phi_{1,1} & \Phi_{1, 2} \\
\Phi_{1,2}^{\T} & \Phi_{2,2}
\end{bmatrix} \begin{bmatrix}
P_1 \\
P_2
\end{bmatrix}
\!\!=\!\!
\begin{bmatrix}
\Psi_{1,1}& 0 \\
0&\Psi_{2,2}
\end{bmatrix} \begin{bmatrix}
P_1 \Lambda_1\\
P_2 \Lambda_2
\end{bmatrix}, \label{eq:frameowrk}
\end{align}
which is  a multivariate eigenvalue problem because in general $\Lambda_1\ne\Lambda_2$.
It is worth noting that \eqref{eq:frameowrk} becomes GEP if $\Lambda_1=\Lambda_2$.

The optimization problem \eqref{op:framework0}
is usually referred to as the MAXBET problem \cite{geer:1984,berg:1988,liww:2015}
and it is numerically challenging. In fact, there is no numerical optimization technique that can
solve it with guarantee. Its associated multivariate eigenvalue problem (\ref{eq:frameowrk})
is notoriously difficult to solve as well, and there is no existing numerical linear algebra technique
that can directly solve it with guarantee, even for the case $k = 1$. Later in Section~\ref{sec:algo},
we will design an efficient successive approximation algorithm to approximately solve \eqref{op:framework0}.

In form, \eqref{op:framework0} differs from CCA \eqref{eq:cca5} in its two extra summands  in \eqref{op:framework}.
Its associated KKT condition \eqref{eq:frameowrk} differs from those semi-paired CCA models in the form of GEP  in
Section~\ref{sec:sp-cca} in that $\Lambda_1\ne\Lambda_2$. Those minorly looking differences have
huge numerical implications. In fact, both CCA and semi-paired CCA models in the form of GEP
can in principle be completely solved by the existing numerical linear algebra techniques \cite{abbd:1999,bddrv:2000,govl:2013,li:2015},
while the numerical states of the art for both \eqref{op:framework0} and \eqref{eq:frameowrk} are unsatisfactorily.

As a demonstration, let us look at why $\Lambda_1 \ne \Lambda_2$ in general.
It follows from \eqref{eq:frameowrk} that
\begin{align*}
\Lambda_1 = P_1^{\T} \Psi_{1,1}P_1 \Lambda_1 = P_1^{\T} \Phi_{1,2} P_2 + P_1^{\T} \Phi_{1,1}P_1, \\
\Lambda_2 = P_2^{\T} \Psi_{2,2}P_2 \Lambda_2 = P_2^{\T} \Phi_{1,2}^{\T} P_1 + P_2^{\T} \Phi_{2,2}P_2.
\end{align*}
Hence $\Lambda_1 = \Lambda_2$ implies $ P_1^{\T} \Phi_{1,1}P_1=P_2^{\T} \Phi_{2,2}P_2$. The latter
is guaranteed true for CCA \eqref{eq:cca5} since $\Phi_{s,s}=\Psi_{s,s}\,\,\forall s$ (in fact both are the same as $C_{s,s}$). But in general for the semi-paired models
in Section~\ref{sec:our-models} below, $\Phi_{s,s}\ne \Psi_{s,s}$.
 
\subsection{New semi-paired models}\label{sec:our-models}
Under the proposed framework (\ref{op:framework0}), we showcase five semi-paired models in both unsupervised and semi-supervised settings. They are motivated from the existing semi-paired models in Section~\ref{sec:sp-cca}, but with 
uncorrelated constraints on extracted features.

\subsubsection{Unsupervised learning}
SemiCCA (\ref{op:semicca-primal}) can be modified to have uncorrelated constraints as
\begin{subequations}\label{op:usemicca}
	\begin{align}
	\max_{P_1, P_2} &~ \gamma \tr(P_1^{\T} C_{1,2} P_2) + \frac{1-\gamma}{2} \sum_{s=1}^2 \tr(P_s^{\T} \widetilde{C}_{s,s} P_s) \label{op:usemicca-1} \\
	\st &~ P_s^{\T} (\gamma C_{s,s} + (1-\gamma) I_{d_s}) P_s = I_k ~~\forall s=1,2, \label{op:usemicca-2}
	\end{align}
\end{subequations}
which falls into the proposed framework (\ref{op:framework0}) with
\begin{subequations}\label{op:usemicca-para}
	\begin{align}
	\Phi_{1,2} &= \gamma C_{1,2},\,\, \Phi_{s,s} = (1-\gamma) \widetilde{C}_{s,s}\,\, \forall s=1,2, \label{op:usemicca-para-1}\\
	\Psi_{s,s} &= \gamma C_{s,s} + (1-\gamma) I_{d_s}\,\, \forall s=1,2. \label{op:usemicca-para-2}
	\end{align}
\end{subequations}
Like SemiCCA \eqref{op:semicca-primal}, model (\ref{op:usemicca}) exactly recovers CCA when $\gamma=1$,
but unlike SemiCCA, it also exactly recovers PCA on $\wtd{C}_{s,s}$ for $s=1,2$, respectively, when $\gamma=0$.
Recall that SemiCCA (\ref{op:semicca-primal}) for $\gamma=0$ is not exactly PCA because
$\Lambda_1\ne\Lambda_2$ in \eqref{eq:frameowrk} in general.

SemiCCALR \eqref{op:semicca-lr} can be adapted to have uncorrelated constraints as
\begin{align}
\max_{P_1, P_2}  \tr(P_1^{\T} C_{1,2} P_2), ~
\st~  P_s^{\T} (C_{s,s} + R_s) P_s = I_k\,\, \forall s, \label{op:usemicca-lr}
\end{align}
which falls into the proposed framework (\ref{op:framework0}) with
\begin{subequations}\label{op:usemicca-lr-para}
	\begin{align}
	\Phi_{1,2} &= C_{1,2},\,\, \Phi_{s,s} = 0\,\, \forall s=1,2, \label{op:usemicca-lr-para-1} \\
	\Psi_{s,s} &=  C_{s,s} +  \gamma_1 I_{d_s} + \gamma_2 \widetilde{X}_s L_s \widetilde{X}_s^{\T}\,\, \forall s=1,2. \label{op:usemicca-lr-para-2}
	\end{align}
\end{subequations}
For ease of reference, we refer to model (\ref{op:usemicca}) as USemiCCA, and model (\ref{op:usemicca-lr}) as USemiCCALR.

\subsubsection{Semi-supervised learning}
CCA \eqref{eq:cca5} is an unsupervised and uncorrelated method.
It can be made to incorporate supervised information, e.g.,
via linear discriminant analysis to maximize the between-class scatter with constrained within-class scatter.
This leads to an uncorrelated semi-supervised CCA (USCCA):
\begin{subequations}\label{op:ucca}
	\begin{align}
	\max_{P_1, P_2} &  \tr(P_1^{\T} C_{1,2} P_2) + \frac{\eta}{2} \sum_{s=1}^2 \tr(P_s^{\T} S_{\rm b}^{(s)}  P_s)\label{op:ucca-1}\\
	\st &~ \eta P_1^{\T} S_{\rm w}^{(1)} P_1 = \eta P_2^{\T} S_{\rm w}^{(2)} P_2 = I_k, \label{op:ucca-2}
	\end{align}
\end{subequations}
which falls into the proposed framework (\ref{op:framework0}) with
\begin{subequations}\label{op:ucca-para}
	\begin{align}
	\Phi_{1,2} &= C_{1,2},\,\, \Phi_{s,s} = \eta S_{\rm b}^{(s)}\,\, \forall s=1,2, \label{op:ucca-para-1} \\
	\Psi_{s,s} &=  \eta S_{\rm w}^{(s)}\,\, \forall s=1,2. \label{op:ucca-para-2}
	\end{align}
\end{subequations}

The adaptation of S$^2$GCA (\ref{op:s2gca-primal}) for uncorrelated constraints can be written as
\begin{subequations}\label{op:us2gca}
	\begin{align}
	\max_{P_1, P_2} &~ \gamma \tr(P_1^{\T} C_{1,2} P_2) \nonumber\\
	&+ \frac{1}{2} \sum_{s=1}^2 \tr(P_s^{\T} \big[\eta S_{\rm b}^{(s)} + (1-\gamma) \widetilde{C}_{s,s}\big] P_s) \label{op:us2gca-1}\\
	\st &~ \eta P_s^{\T}  S_{\rm w}^{(s)}P_s + (1-\gamma ) P_s^{\T}P_s= I_k\,\, \forall s=1,2, \label{op:us2gca-2}
	\end{align}
\end{subequations}
which falls into the proposed framework (\ref{op:framework0}) with
\begin{subequations}\label{op:us2gca-para}
	\begin{align}
	\Phi_{1,2} &= \gamma C_{1,2}, \,\Phi_{s,s} = \eta S_{\rm b}^{(s)} + (1-\gamma) \widetilde{C}_{s,s}\,\, \forall s=1,2, \label{op:us2gca-para-1}\\
	\Psi_{s,s} &= \eta S_{\rm w}^{(s)} + (1-\gamma) I_{d_s}\,\, \forall s=1,2.\label{op:us2gca-para-2}
	\end{align}
\end{subequations}
It is worth noting that our approach to leveraging supervised information in \eqref{op:us2gca} is different from  (\ref{op:s2gca-primal}) where the entire
$\eta (S_{\rm b}^{(s)} - S_{\rm w}^{(s)})$ appears in the objective but here it is broken into two with
$\eta S_{\rm b}^{(s)}$ still in the objective while $\eta  S_{\rm w}^{(s)}$ showing up in the constraints as for LDA.

In addition, SemiCCALR \eqref{op:semicca-lr} can be used as the base model for incorporating both supervised information and the uncorrelated constraints to give
\begin{subequations}\label{op:usemi2cca-lr}
	\begin{align}
	\max_{P_1, P_2} &~ \tr(P_1^{\T} C_{1,2} P_2) + \frac{\eta}{2} \sum_{s=1}^2 \tr(P_s^{\T} S_{\rm b}^{(s)}   P_s)
	\label{op:usemi2cca-lr-1} \\
	\st &~  \eta P_s^{\T} (S_{\rm w}^{(s)} + R_s) P_s = I_k\,\, \forall s=1,2. \label{op:usemi2cca-lr-2}
	\end{align}
\end{subequations}
Again this formulation also falls into the proposed framework (\ref{op:framework0}) with
\begin{subequations}\label{op:usemi2cca-lr-para}
	\begin{align}
	\Phi_{1,2} &= C_{1,2},\,\,\Phi_{s,s} = \eta S_{\rm b}^{(s)} \,\, \forall s=1,2, \label{op:usemi2cca-lr-para-1} \\
	\Psi_{s,s} &=  \eta S_{\rm w}^{(s)} +  \gamma_1 I_{d_s} + \gamma_2 \widetilde{X}_s L_s \widetilde{X}_s^{\T}\,\, \forall s=1,2.
	\label{op:usemi2cca-lr-para-2}
	\end{align}
\end{subequations}
We will refer to model \eqref{op:us2gca} as US$^2$GCA and model (\ref{op:usemi2cca-lr}) as US$^2$CCALR.

\section{Successively Alternating Approximation (SAA)}\label{sec:algo}
Note that the framework \eqref{op:framework0} and its instantiated novel models bear the same optimization formulation
\begin{subequations}\label{op:Uframework}
	\begin{equation}\label{op:Uframework-1}
	\max_{P_s^{\T}B_s P_s=I_{k}\,\forall s}\,\,f(P_1,P_2),
	\end{equation}
	where
	\begin{equation}\label{op:Uframework-2}
	f(P_1,P_2):={\tr(P_1^{\T} C P_2) }
	+ \frac12\sum_{s=1}^2 { \tr(P_s^{\T} A_s P_s)},
	\end{equation}
\end{subequations}
$A_s\in \bbR^{d_s\times d_s}$  are symmetric, $B_s\in \bbR^{d_s\times d_s}$  are symmetric positive  definite and $P_s\in \bbR^{d_s\times k}$ for $s=1,2$. Let
\begin{equation}\label{eq:scrA0}
\scrA =\left[\begin{array}{cc}A_1 & C \\ C^{\T} & A_2\end{array}\right],~~P=\left[\begin{array}{c}P_1 \\P_2\end{array}\right].
\end{equation}
Then we have $f(P_1,P_2)=\frac12\tr(P^{\T}\scrA P)$.

We start by transforming \eqref{op:Uframework} into the case $B_s=I_{d_s}$ for $s=1,2$. Let
$B_s=L_sL_s^{\T}$ be the Cholesky decompositions  and set
\begin{align}
\bar P&=\diag(L_1^{\T},L_2^{\T})P, \label{eq:P2barP}\\
\bar \scrA &=\diag(L_1^{-1},L_2^{-1})\scrA \diag(L_1^{-\T},L_2^{-\T}). \label{eq:A2barA}
\end{align}
Then \eqref{op:Uframework} is turned into
\begin{equation}\label{op:Uframework2}
\argmax_{\bar P_s\in \bbO^{d_s\times k}\,\forall s} \tr(\bar P^{\T} \bar \scrA  \bar P).
\end{equation}
The optimizers of \eqref{op:Uframework} and \eqref{op:Uframework2} are related according to \eqref{eq:P2barP}.

Problem \eqref{op:Uframework2} is an MAXBET problem \cite{geer:1984,berg:1988,liww:2015}. Unfortunately, except for trivial cases (such as $A_s=0$ in CCA \eqref{eq:cca5}, or $B_s=A_s$ \cite{chwa:1993}), \eqref{op:Uframework2} does not admit a closed form solution. Moreover, there are no efficient solvers that can guarantee to compute its global maximizer, and
existing optimization methods are too expensive to handle large scale ones.
In what follows, we will propose a successive alternating approximation scheme to solve \eqref{op:Uframework2} by
building one column of $\bar P$ at a time.
Although the new scheme still doesn't guarantee that the computed solution is globally optimal, it
admits the following advantages that no existing method does:
\begin{enumerate}[(a)]
	\item for the trivial cases $A_s=0$ or $B_s=A_s$ for $s=1,2$, or $k=1$, it finds the global maximizer $\bar P$ of \eqref{op:Uframework2} and thus of \eqref{op:Uframework} by extension;
	\item efficient and scalable Krylov subspace methods can be readily exploited for large scale problems.
\end{enumerate}
 
\subsection{Algorithmic framework}\label{subsec:SDA}
Recall that we will solve \eqref{op:Uframework2} and then recover a solution to \eqref{op:Uframework}
according to \eqref{eq:P2barP}, i.e., $P=\diag(L_1^{-\T},L_2^{-\T})\bar P$.
For conciseness, we will drop all the bars in notation, or equivalently assume $B_s=I_{d_s}$ for $s=1,2$
in Sections~\ref{subsec:SDA} to \ref{subsec:deflation}.
Finally in Section~\ref{subsec:SDAvTRS},  we present our final complete algorithm for general $B_s\ne I_{d_s}$.

Our scheme is similar to that for computing the top $k$ principal component vectors in PCA and the top $k$ canonical correlation vectors in CCA \cite[Section 14.1]{hasi:2007}. It starts by calculating
the first column vector $\bp^{(1)}$ of optimal $ P$ via
\begin{align}\label{eq:1stpair}
\bp^{(1)}=\argmax_{\|\bp_s\|_2=1\,\forall s}  \bp^{\T} \scrA  \bp,
\end{align}
where and henceforth $\bp_s\in\bbR^{d_s}$ is implicitly assumed to be the subvectors of $\bp\in\bbR^{d_1+d_2}$
partitioned as $\bp=[\bp_1^{\T},\bp_2^{\T}]^{\T}$. To solve \eqref{eq:1stpair}, we adopt
an alternating approximation scheme to maximize $ \bp^{\T} \scrA  \bp$ alternatingly between
$\bp_1$ and $\bp_2$ as detailed in the next subsection.

Suppose now approximations to the first $j$ columns of $P_s$ are gotten:
$$
\bp^{(i)}=\begin{bmatrix}
\bp_1^{(i)} \\ \bp_2^{(i)}
\end{bmatrix},\,\bp_s^{(i)}\in \bbR^{d_s}\,\,\mbox{for}\,\,i=1,2,\ldots,j, 
$$
and set
$$
P_s^{(j)}=[\bp_s^{(1)},\dots,\bp_s^{(j)}]\in \bbO^{d_s\times j},\quad s=1,2.
$$
The $(j+1)$st column $\bp^{(j+1)}$ is then determined by
\begin{align}\label{eq:jthpair}
\bp^{(j+1)}=\argmax_{
	\|\bp_s\|_2=1,~\bp_s^{\T} P_s^{(j)}=\mathbf{0}\,\forall s}  \bp^{\T} \scrA  \bp.
\end{align}
This problem will be again solved alternatingly.
It is not hard to see that the resulting $ P_s^{(j+1)}=[ P_s^{(j)},\bp_s^{(j+1)}]\in \bbO^{d_s\times (j+1)}$ for $s=1,2$.
The procedure stops until after $\bp^{(k)}$ is computed. 

We name the whole procedure the Successively Alternating Approximation (SAA). In the next two subsections,
we will explain how to solve \eqref{eq:1stpair} and \eqref{eq:jthpair}.
 
\subsection{Maximization by an alternating scheme}\label{subsec:alter}
We will explain how to solve \eqref{eq:1stpair} in this subsection and then in the next subsection
we show how to turn \eqref{eq:jthpair} into one in the form of \eqref{eq:1stpair}.

For ease of reference later, we will use slightly different notations for \eqref{eq:1stpair}:
\begin{equation}\label{eq:scrA-generic}
\scrA^{\T}=\scrA=\kbordermatrix{ &\sss \tilde d_1 &\sss \tilde d_2 \\
	\sss \tilde d_1 & A_{11} & A_{12} \\
	\sss \tilde d_2 & A_{21} & A_{22} }, \,\,
\bp=\kbordermatrix{ & \\
	\sss \tilde d_1 & \bp_1 \\
	\sss \tilde d_2 & \bp_2 }
\end{equation}
so that later we can call it with different $\tilde d_s$ as $d_s$ has been reserved.

We will solve \eqref{eq:1stpair} by maximizing $\bp^{\T}A\bp$ alternatingly between $\bp_1$ and $\bp_2$
by fixing one at the current approximation and maximizing over the other until convergence. 
Specifically, it goes as follows: given an approximation $\bp_2^{(1)}$ (or simply taking a random one),
repeat
\begin{subequations}\label{eq:subp:k=1}
	\begin{align}
	\bp_1^{(1)}&=\argmax_{\|\bp_1\|_2=1}\, \bp_1^{\T}A_{11}\bp_1+2(A_{21}\bp_2^{(1)})^{\T}\bp_1, \label{eq:subp:k=1a}\\
	\bp_2^{(1)}&=\argmax_{\|\bp_2\|_2=1}\, \bp_2^{\T}A_{22}\bp_2+2(A_{12}\bp_1^{(1)})^{\T}\bp_2, \label{eq:subp:k=1b}
	\end{align}
\end{subequations}
until convergence.
Both are in the form of the well-known trust-region subproblem (TRS) for which very efficient methods have been proposed for both small and large scale  problems.
 
TRS is one of the most well-studied optimization problems \cite{cogt:2000,nowr:2006}.
Theoretically, sufficient and necessary optimality conditions for
the global solution were developed by Gay \cite{gay:1981} and Mor\'e and Sorensen \cite{moso:1983}  (see also \cite{sore:1982} and \cite[Theorem 4.1]{nowr:2006}), and numerically, there are efficient methods that can guarantee to compute a global maximizer.
In particular,  the  Mor\'e-Sorensen method \cite{moso:1983} is a  Newton method that solves its KKT system and it is efficient  for small to medium sized TRS.  For large scale TRS, several efficient numerical approaches can be used  (see, e.g.,
\cite{hage:2001,rewo:1997,ross:2008,stei:1983}). Here we mention the Krylov subspace type method, namely the {\it  Generalized Lanczos Trust-Region} (\texttt{GLTR}) method \cite{golr:1999} (see also \cite[Chapter 5]{cogt:2000})
because of its popularity. Although {\tt GLTR} was developed two decades ago, its complete convergence analysis
is  more of recent works \cite{zhsl:2017,cadu:2018}, along with some improvements \cite{zhsh:2018}.

\begin{algorithm}[thb!!!]
	\caption{Alternating Approximation for \eqref{eq:1stpair}}
	\label{alg:TRS}
	\begin{flushleft} 
	{\bf Input}: symmetric $\scrA$ partitioned as in \eqref{eq:scrA-generic}; \\
	{\bf Output}: approximate solution $(\bp_1^{(1)},\bp_2^{(1)})$ with $\bp_s\in\bbR^{d_s}$.
\end{flushleft}
	\hrule
	\begin{algorithmic}[1]
		\STATE choose an initial guess for $\bp^{(1)}_2$;
		\REPEAT
		\STATE solve \eqref{eq:subp:k=1a} for its maximizer, by either \texttt{trust} or \texttt{GLTR}; 
		\STATE solve \eqref{eq:subp:k=1b} for its maximizer, by either \texttt{trust} or \texttt{GLTR}; 
		\UNTIL convergence
		\RETURN the last $(\bp_1^{(1)},\bp_2^{(1)})$.
	\end{algorithmic}
\end{algorithm}

In our numerical experiments, we use MATLAB's built-in function
\texttt{trust}\footnote{MATLAB's \texttt{trust} is available in MATLAB 7.0 (R14). It computes
	the full eigen-decomposition of the involved matrix and then solves the resulting secular equation.
	Hence, \texttt{trust} is only suitable for small-to-medium sized TRS.}
whenever the size $d\le 500$. For $d>500$, {\tt GLTR} is called. Here $d=\tilde d_1$ or $\tilde d_2$, depending on
which one of \eqref{eq:subp:k=1a} and \eqref{eq:subp:k=1b} is being solved. Algorithm~\ref{alg:TRS} summarizes
the algorithm for \eqref{eq:1stpair}.

\subsection{Transform (52)}\label{subsec:deflation} 
Let $U_s^{(j)}\in \bbO^{d_s\times (d_s-j)}$ such that $[ P_s^{(j)},U_s^{(j)}]\in \bbO^{d_s\times d_s}$, i.e.,
orthogonal. Then $\cR(U_s^{(j)})=\cR( P_s^{(j)})^{\perp}$. Any $\bp_s$ such that $\bp_s^{\T} P_s^{(j)}=0$
is in $\cR( P_s^{(j)})^{\perp}$ and vice versa, and hence $\bp_s=U_s^{(j)}\bq_s$ for some $\bq_s\in\bbR^{d_s-j}$
and $\|\bp_s\|_2=\|\bq_s\|_2$. Consequently,
$$
\bp^{\T} \scrA  \bp=\bq^{\T}\what\scrA\bq,
$$
where $\bq=[\bq_1^{\T},\bq_2^{\T}]^{\T}$ and
$$
\what\scrA=\kbordermatrix{ &\sss d_1-j &\sss d_2-j \\
	\sss d_1-j & (U_1^{(j)})^{\T}A_{11}U_1^{(j)} & (U_1^{(j)})^{\T}A_{12}U_2^{(j)} \\
	\sss d_2-j & (U_2^{(j)})^{\T}A_{21}U_1^{(j)} & (U_2^{(j)})^{\T}A_{22}U_2^{(j)} }.
$$
We have proved the following theorem.

\begin{theorem}\label{thm:trans-jpair}
	Problem \eqref{eq:jthpair} is equivalent to
	\begin{equation}\label{eq:trans-jthpair}
	\bq^{\opt}:=\argmax_{\|\bq_s\|_2=1\,\forall s} \bq^{\T}\what\scrA\bq
	\end{equation}
	in the sense that $\bp_s^{(j+1)}=U_s^{(j)}\bq_s^{\opt}\,\forall s$ and $\bq_s^{\opt}=(U_s^{(j)})^{\T}\bp_s^{(j+1)}\,\forall s$.
\end{theorem}

The transformed problem \eqref{eq:trans-jthpair} takes exactly the same form as \eqref{eq:1stpair}, and thus can be solved
in the same way as described in Section~\ref{subsec:alter}.

It remains to explain how to construct $U_s^{(j)}\in \bbO^{d_s\times (d_s-j)}$ and $\what\scrA$. Theoretically, they can be extracted from the $Q$-factor of the full $QR$ decomposition of $P_s^{(j)}$ as shown by Theorem~\ref{thm:orth-comp} below.

\begin{theorem}\label{thm:orth-comp}
	Let the $QR$ decomposition of $ P_s^{(j)}$ be
	\begin{equation}\label{eq:Ps(j)-QR}
	P_s^{(j)}=Q_s^{(j)}R_s^{(j)},\,\, Q_s^{(j)}\in\bbO^{d_s\times d_s}.
	\end{equation}
	Then $U_s^{(j)}=\big[Q_s^{(j)}\big]_{(:,j+1:d_s)}$, i.e., the last $d_s-j$ columns of $Q_s^{(j)}$.
\end{theorem}

\begin{proof}
	Since $ P_s^{(j)}\in\bbO^{d_s\times j}$ and $R_s^{(j)}\in\bbR^{d_s\times j}$ is upper triangular, $\big[R_s^{(j)}\big]_{(1:j,:)}$
	must be nonsingular (in fact, it can be made $I_j$). Hence $\cR( P_s^{(j)})=\cR(\big[Q_s^{(j)}\big]_{(:,1:j)})$ and
	$\cR( P_s^{(j)})^{\perp}=\cR(\big[Q_s^{(j)}\big]_{(:,j+1:d_s)})$.
\end{proof}

Numerically, we will not compute the $QR$ decomposition of $ P_s^{(j)}$ every single time when $j$ is increased by $1$, but
rather keep $Q_s^{(j)}$ in the product form of $j$ elementary orthogonal matrices -- Householder matrices \cite{demm:1997,govl:2013} in our implementation. Accordingly, there is no need to form $U_s^{(j)}$ explicitly, and $(U_s^{(j)})^{\T}A_{st}U_t^{(j)}$ will be
updated from the previous $(U_s^{(j-1)})^{\T}A_{st}U_t^{(j-1)}$ efficiently.

For $j=1$, we compute the Householder matrix $H_s^{(1)}=I_{d_s}-2\bu_s^{(1)}(\bu_s^{(1)})^{\T}$ such that
$H_s^{(1)}\bp_s^{(1)}=\alpha_1 \be_1$ to give $\bp_s^{(1)}=H_s^{(1)} (\alpha_1\be_1)$ and $Q_s^{(1)}=H_s^{(1)}$.
It is based on the following well-known fact.

\begin{theorem}\label{thm:householder}
	For any vector $\by\in\bbR^d$ that is not a scalar multiple of $\be_1$, let
	$$
	\alpha=-\sign(\by_{(1)})\|\by\|_2, \,\,\bu=\frac {\by-\alpha\be_1}{\|\by-\alpha\be_1\|_2},
	$$
	where $\sign(\by_{(1)})$ is the sign of $\by_{(1)}$, i.e., $1$ if $\by_{(1)}\ge 0$, and $0$ otherwise.
	Then $H\by=\alpha\be_1$, where Householder matrix $H=I_d-2\bu\bu^{\T}$.
\end{theorem}

The reader is referred to \cite{demm:1997,govl:2013} or any other books on matrix computations
for more detail.
We will emphasize that it suffices to just store $\bu_s^{(1)}\in\bbR^{d_s}$ for $H_s^{(1)}$.
Accordingly,
$$
(U_s^{(1)})^{\T}A_{st}U_t^{(1)}=\big[H_s^{(1)}A_{st}H_t^{(1)}\big]_{(2:d_s,2:d_s)}
$$
can be compute efficiently in $O(d_s^2)$ flops.

In general, we have
\begin{equation}\label{eq:Qsj-form}
Q_s^{(j)}=H_s^{(1)}\,\diag(1,H_s^{(2)})\,\cdots \diag(I_{j-1},H_s^{(j)}),
\end{equation}
where $H_s^{(j)}=I_{d_s-j+1}-2\bu_s^{(j)}(\bu_s^{(j)})^{\T}$ with $\bu_s^{(j)}\in\bbR^{d_s-j+1}$.
In form, we have
\eqref{eq:Ps(j)-QR} but neither $Q_s^{(j)}$ nor $R_s^{(j)}$ is explicitly computed; only the existence of
$Q_s^{(j)}$ in the form of \eqref{eq:Qsj-form} matters. After $\bq_s^{\opt}$ as defined in Theorem~\ref{thm:trans-jpair}
is computed,
\begin{equation}\label{eq:ps(j+1):comp}
\bp_s^{(j+1)}=U_s^{(j)}\bq_s^{\opt}=\big[Q_s^{(j)}\big]_{(:,j+1:d_s)}\bq_s^{\opt}
\end{equation}
can be done in $O(jd_s)$ flops.
Suppose now $\bp_s^{(j+1)}$ has just been computed. We then have
\begin{align*}
P_s^{(j+1)}&=[ P_s^{(j)},\bp_s^{(j+1)}]=[Q_s^{(j)}R_s^{(j)},\bp_s^{(j+1)}]\\
&=Q_s^{(j)}[R_s^{(j)},(Q_s^{(j)})^{\T}\bp_s^{(j+1)}]\\
&=:Q_s^{(j)}\begin{bmatrix}
R_s^{(j)} & 0 \\ 0& \br_s^{(j+1)}
\end{bmatrix}.
\end{align*}
It follows from \eqref{eq:ps(j+1):comp} that $\br_s^{(j+1)}=\bq_s^{\opt}$. Next, we find Householder matrix
$$
H_s^{(j+1)}=I_{d_s-j}-2\bu_s^{(j+1)}(\bu_s^{(j+1)})^{\T}\in \bbO^{(d_s-j)\times (d_s-j)}
$$
such that $H_s^{(j+1)}\br_s^{(j+1)}=\alpha_{j+1}\be_1$ to yield $Q_s^{(j+1)}$ by adding another matrix-factor
$\diag(I_j,H_s^{(j+1)})$ to the right end of the expression for
$Q_s^{(j)}$ in \eqref{eq:Qsj-form}, and then
\begin{multline}\label{eq:UAU-update}
(U_s^{(j+1)})^{\T}A_{st}U_t^{(j+1)}\\
=\big[H_s^{(j+1)}\{(U_s^{(j)})^{\T}A_{st}U_t^{(j)}\}H_t^{(j+1)}\big]_{(2:d_s-j,2:d_s-j)}
\end{multline}
in $O(d_s^2)$ flops.

\begin{algorithm}[t] 
	\caption{Successively Alternating Approximation (SAA)}
	\label{alg:SAA}
\begin{flushleft}
	{\bf Input}: data matrices $\scrA$ as in \eqref{eq:scrA0} and $\{B_s\}_{s=1}^2$; \\  
	{\bf Output}:  approximation solution pair $(P_1,P_2)$ of \eqref{op:Uframework}.
\end{flushleft}
	\hrule
	\begin{algorithmic}[1]
		\STATE compute Cholesky decompositions $B_s=L_sL_s^{\T}$ for $s=1,2$;
		\STATE $\scrA \leftarrow \diag(L_1^{-1},L_2^{-1})\scrA \diag(L_1^{-\T},L_2^{-\T})$, partitioned as $[A_{st}]$
		with $A_{st}\in\bbR^{d_s\times d_t}$ for $s,t\in\{1,2\}$;
		\STATE call Algorithm~\ref{alg:TRS} with input $\scrA=[A_{st}]_{s,t=1}^2$ to yield output $(\bq_1^{\opt},\bq_2^{\opt})$;
		\STATE $\bp_s^{(1)}=\bq_s^{\opt}$ for $s=1,2$;
		\FOR{$j=1,2,\dots,k-1$}
		\STATE construct Householder matrix $H_s^{(j)}=I-2\bu_s^{(j)}(\bu_s^{(j)})^{\T}$ on $\bq_s^{\opt}$ according to
		Theorem~\ref{thm:householder};
		\STATE $A_{st}\leftarrow [H_s^{(j)}A_{st}H_s^{(j)}]_{(2:d_s-j+1,2:d_s-j+1)}$ for $s,t\in\{1,2\}$;
		\STATE call Algorithm~\ref{alg:TRS} with input $\scrA=[A_{st}]_{s,t=1}^2$ to yield output $(\bq_1^{\opt},\bq_2^{\opt})$;
		\STATE $\bp_s^{(j+1)}=\big[Q_s^{(j)}\big]_{(:,j+1:d_s)}\bq_s^{\opt}$ for $s=1,2$;
		\ENDFOR
		\STATE $P_s=[\bp_s^{(1)},\ldots,\bp_s^{(k)}]$ for $s=1,2$;
		\STATE $P_s\leftarrow L_s^{-\T}P_s$ for $s=1,2$;
		\STATE compute SVD: $P_1^{\T}CP_2=U\Sigma V^{\T}$, and set $P_2\leftarrow P_2VU^{\T}$;
		\RETURN $(P_1,P_2)$.
	\end{algorithmic}
\end{algorithm}

\subsection{The complete algorithm}\label{subsec:SDAvTRS}

Our complete algorithm is outlined in Algorithm~\ref{alg:SAA} whose line 13 is to properly align $P_1$ and $P_2$ computed
up to line 12 in such a way that the first term in the objective \eqref{op:Uframework-2} is made increasing  while the last
two summands remain the same:
$$
P_1^{\T}CP_2=U\Sigma V^{\T}
\quad\Rightarrow\quad
P_1^{\T}C(P_2VU^{\T})=U\Sigma U^{\T},
$$
based on a result from matrix analysis \cite[Lemma 3]{zhwb:2020}.

\section{Experiments} \label{sec:experiments}

\subsection{Multiple feature data}
Multiple features (mfeat) dataset consists of features of handwritten numerals (`0'--`9') extracted from a collection of Dutch utility maps\footnote{https://archive.ics.uci.edu/ml/datasets/Multiple+Features}, in which 200 patterns per class (for a total of 2,000 patterns) have been digitized in binary images. These digits are represented in terms of the following six feature sets: 216-dim profile correlations (fac), 76-dim Fourier coefficients of the character shapes (fou),  64-dim Karhunen-Love coefficients (kar), 6-dim morphological features (mor),
240-dim pixel averages in $2 \times 3$ windows (pix), and 47-dim Zernike moments (zer). As a result, there are 6 views.

\begin{table}
	\setlength{\tabcolsep}{1pt}
	\caption{Average accuracy with standard deviation by five unsupervised semi-paired learning methods evaluated on data mfeat over $10$ randomly drawn training and testing splits. The best results are in bold.} \label{tab:mfeat-unsup}
	\centering
	\vspace{-0.1in}
	\begin{tabular}{@{}c|c|c|c|c|c@{}}
		\hline
		v1-v2& CCA& SemiCCA& USemiCCA&\scriptsize SemiCCALR&\scriptsize USemiCCALR\\ 
		\hline
		fac-fou & 57.84 $\pm$ 3.47 & 91.99 $\pm$ 1.18 & 93.54 $\pm$ 1.04 & \textbf{94.64 $\pm$ 1.03} & \textbf{94.64 $\pm$ 1.03}\\
		fac-kar & 78.25 $\pm$ 2.90 & 82.85 $\pm$ 2.79 & 88.40 $\pm$ 1.95 & 89.28 $\pm$ 1.82 & \textbf{89.44 $\pm$ 1.36}\\
		fac-mor & 44.12 $\pm$ 4.48 & 82.35 $\pm$ 0.95 & 90.40 $\pm$ 1.09 & 92.29 $\pm$ 1.63 & \textbf{92.38 $\pm$ 1.55}\\
		fac-pix & 75.74 $\pm$ 1.78 & 86.02 $\pm$ 2.43 & 88.51 $\pm$ 1.60 & \textbf{90.37 $\pm$ 1.67} & \textbf{90.37 $\pm$ 1.69}\\
		fac-zer & 56.09 $\pm$ 3.03 & 79.60 $\pm$ 3.36 & 85.96 $\pm$ 1.70 & 90.13 $\pm$ 1.22 & \textbf{90.15 $\pm$ 1.18}\\
		fou-kar & 61.64 $\pm$ 4.02 & 92.53 $\pm$ 0.78 & \textbf{94.06 $\pm$ 1.04}& 93.46 $\pm$ 1.17 & 93.46 $\pm$ 1.17\\
		fou-mor & 68.43 $\pm$ 3.57 & 80.11 $\pm$ 1.71 &\textbf{80.92 $\pm$ 1.70} & 79.70 $\pm$ 1.57 & 79.86 $\pm$ 1.47\\
		fou-pix & 65.31 $\pm$ 3.13 & 90.91 $\pm$ 1.54 & 91.89 $\pm$ 1.43 & \textbf{93.65 $\pm$ 1.30}& \textbf{93.65 $\pm$ 1.30}\\
		fou-zer & 63.36 $\pm$ 3.75 & 78.15 $\pm$ 0.52 & 81.93 $\pm$ 1.41 & 83.07 $\pm$ 1.32 & \textbf{83.25 $\pm$ 1.22}\\
		kar-mor & 72.31 $\pm$ 2.30 & 89.18 $\pm$ 2.14 & \textbf{92.49 $\pm$ 1.31} & 92.11 $\pm$ 0.85 & 92.01 $\pm$ 1.26\\
		kar-pix & 84.24 $\pm$ 1.58 & 85.27 $\pm$ 1.46 & \textbf{88.19 $\pm$ 1.59} & 87.97 $\pm$ 1.81 & 87.97 $\pm$ 1.81\\
		kar-zer & 63.08 $\pm$ 3.23 & 87.07 $\pm$ 1.45 & 88.66 $\pm$ 1.34 & 89.84 $\pm$ 0.93 & \textbf{89.85 $\pm$ 1.06}\\
		mor-pix & 48.03 $\pm$ 3.71 & 87.03 $\pm$ 1.24 & 87.03 $\pm$ 1.22 & 91.60 $\pm$ 1.08 & \textbf{91.68 $\pm$ 0.98}\\
		mor-zer & 70.17 $\pm$ 3.03 & 72.19 $\pm$ 1.98 & 73.77 $\pm$ 1.49 & 77.12 $\pm$ 1.54 & \textbf{77.49 $\pm$ 1.59}\\
		pix-zer & 56.26 $\pm$ 2.73 & 84.44 $\pm$ 1.83 & 86.01 $\pm$ 2.31 & \textbf{90.11 $\pm$ 1.29} & \textbf{90.11 $\pm$ 1.29}\\
		\hline
	\end{tabular}\vspace{-0.2in}
\end{table}

Following the semi-paired data generation process in \cite{chen2012unified}, we perform experiments on datasets of any pair of the $6$ views, a total of $15$ two-view datasets.
For each two-view dataset, we randomly select $50\%$ of the data for training  and the rest for testing.
Among the training data, $20\%$ data are randomly selected as paired  and the rest as unpaired.
For semi-supervised learning, we randomly sample $10\%$ of the training data as labeled  and the rest as unlabeled.
The nearest neighbor classifier (NNC) is used to evaluate the  projection matrices learned by  each compared method. 
The concatenation of projected points of both two views are evaluated.
NNC is trained on the training data, and then it is assessed on the testing data.
For unsupervised learning, all  training data are assumed without any label.
We repeat each experiment $10$ times, following the above semi-paired data generation process, and then report
its average classification accuracy with  standard deviation.

\begin{table*}[t]
	\setlength{\tabcolsep}{5pt}
	\caption{Average accuracy with standard deviation by six methods on 15 datasets derived from mfeat over 10 randomly drawn training and testing splits with three different scatter constructions. The best results are in bold.} \label{tab:sssl-mfeat}
	\centering
	\vspace{-0.1in}
	\begin{tiny}
	\begin{tabular}{c|c|cccccc}
		\hline
		graph construction	&view 1 - view 2& SCCA& USCCA& S$^2$GCA& US$^2$GCA& S$^2$CCALR& US$^2$CCALR\\ \hline
		\multirow{15}{*}{LDA}& fac-fou & 68.10 $\pm$ 2.49 & 92.57 $\pm$ 1.47 & 91.17 $\pm$ 1.59 & 93.58 $\pm$ 0.98 & 94.26 $\pm$ 1.01 & \textbf{94.37 $\pm$ 0.92}\\
		& fac-kar & 84.09 $\pm$ 2.30 & 90.20 $\pm$ 2.07 & 88.01 $\pm$ 2.45 & 91.47 $\pm$ 1.71 & 90.35 $\pm$ 2.55 & \textbf{92.60 $\pm$ 1.30}\\
		& fac-mor & 65.69 $\pm$ 4.52 & 93.75 $\pm$ 1.16 & 86.82 $\pm$ 1.46 & 94.33 $\pm$ 0.84 & \textbf{95.01 $\pm$ 1.04} & 93.91 $\pm$ 1.08\\
		& fac-pix & 30.70 $\pm$ 2.56 & 90.40 $\pm$ 1.05 & 88.63 $\pm$ 1.54 & 90.94 $\pm$ 1.74 & 89.67 $\pm$ 1.40 & \textbf{91.65 $\pm$ 1.37}\\
		& fac-zer & 68.58 $\pm$ 3.05 & 91.33 $\pm$ 1.63 & 85.77 $\pm$ 2.48 & \textbf{93.22 $\pm$ 1.56} & 89.51 $\pm$ 1.59 & 92.59 $\pm$ 1.36\\
		& fou-kar & 74.87 $\pm$ 2.87 & 89.74 $\pm$ 2.41 & 92.00 $\pm$ 1.36 & \textbf{94.18 $\pm$ 0.90} & 92.35 $\pm$ 0.85 & 93.55 $\pm$ 0.79\\
		& fou-mor & 71.03 $\pm$ 1.76 & 77.61 $\pm$ 2.09 & 78.81 $\pm$ 2.01 & \textbf{81.03 $\pm$ 1.17} & 80.61 $\pm$ 1.55 & 80.77 $\pm$ 1.27\\
		& fou-pix & 24.76 $\pm$ 3.17 & 88.54 $\pm$ 2.56 & 93.33 $\pm$ 1.12 & \textbf{94.53 $\pm$ 0.73} & 92.77 $\pm$ 0.73 & 93.58 $\pm$ 1.01\\
		& fou-zer & 71.96 $\pm$ 2.64 & 79.38 $\pm$ 0.93 & 78.95 $\pm$ 0.85 & 81.34 $\pm$ 1.01 & \textbf{81.65 $\pm$ 1.06} & 81.32 $\pm$ 0.68\\
		& kar-mor & 83.40 $\pm$ 2.08 & 90.17 $\pm$ 1.81 & 89.70 $\pm$ 1.94 & 93.12 $\pm$ 0.75 & \textbf{95.07 $\pm$ 0.52} & 94.15 $\pm$ 1.02\\
		& kar-pix & 34.32 $\pm$ 2.83 & 82.30 $\pm$ 2.34 & 88.69 $\pm$ 1.68 & \textbf{90.96 $\pm$ 1.51} & 90.70 $\pm$ 1.48 & 90.84 $\pm$ 1.43\\
		& kar-zer & 78.29 $\pm$ 2.18 & 85.24 $\pm$ 2.83 & 85.74 $\pm$ 1.45 & 90.30 $\pm$ 1.52 & 89.46 $\pm$ 1.32 & \textbf{91.14 $\pm$ 0.94}\\
		& mor-pix & 42.97 $\pm$ 3.40 & 91.82 $\pm$ 1.38 & 90.91 $\pm$ 1.60 & 93.19 $\pm$ 1.09 & \textbf{94.74 $\pm$ 0.86} & 94.66 $\pm$ 1.45\\
		& mor-zer & 75.46 $\pm$ 1.86 & 78.04 $\pm$ 1.71 & 76.24 $\pm$ 1.00 & 78.56 $\pm$ 1.44 & 80.28 $\pm$ 1.66 & \textbf{80.29 $\pm$ 1.75}\\
		& pix-zer & 38.26 $\pm$ 3.37 & 88.65 $\pm$ 1.50 & 87.65 $\pm$ 1.56 & 89.97 $\pm$ 1.63 & 90.06 $\pm$ 1.74 & \textbf{91.36 $\pm$ 1.21}\\
		\hline
		\multirow{15}{*}{LFDA}& fac-fou & 67.22 $\pm$ 4.37 & 93.35 $\pm$ 1.20 & 90.36 $\pm$ 1.45 & 93.56 $\pm$ 0.90 & 93.19 $\pm$ 1.29 & \textbf{93.79 $\pm$ 1.02}\\
		& fac-kar & 84.67 $\pm$ 2.87 & 89.36 $\pm$ 1.65 & 90.09 $\pm$ 1.90 & 90.88 $\pm$ 1.49 & 91.18 $\pm$ 1.26 & \textbf{92.31 $\pm$ 1.15}\\
		& fac-mor & 66.00 $\pm$ 2.10 & 92.78 $\pm$ 0.93 & 90.77 $\pm$ 1.83 & \textbf{94.04 $\pm$ 1.22} & 92.65 $\pm$ 0.99 & 93.19 $\pm$ 1.20\\
		& fac-pix & 82.51 $\pm$ 2.03 & 88.91 $\pm$ 2.00 & 88.00 $\pm$ 2.12 & 91.45 $\pm$ 1.93 & 90.62 $\pm$ 1.55 & \textbf{91.46 $\pm$ 1.13}\\
		& fac-zer & 68.66 $\pm$ 3.00 & 91.23 $\pm$ 1.64 & 91.16 $\pm$ 1.96 & \textbf{91.59 $\pm$ 1.85} & 88.18 $\pm$ 1.58 & 91.11 $\pm$ 1.86\\
		& fou-kar & 77.49 $\pm$ 4.06 & 92.05 $\pm$ 0.93 & 92.34 $\pm$ 1.38 & \textbf{94.38 $\pm$ 0.73} & 91.32 $\pm$ 1.29 & 91.74 $\pm$ 1.03\\
		& fou-mor & 69.97 $\pm$ 2.46 & 79.16 $\pm$ 1.50 & 78.98 $\pm$ 1.34 & \textbf{80.75 $\pm$ 1.67} & 78.24 $\pm$ 1.67 & 78.63 $\pm$ 1.73\\
		& fou-pix & 59.33 $\pm$ 4.14 & 90.40 $\pm$ 1.65 & 93.03 $\pm$ 0.94 & \textbf{94.54 $\pm$ 0.90} & 91.88 $\pm$ 0.93 & 91.96 $\pm$ 1.01\\
		& fou-zer & 74.96 $\pm$ 1.84 & 80.57 $\pm$ 0.81 & 78.92 $\pm$ 1.24 & \textbf{81.50 $\pm$ 1.32} & 81.46 $\pm$ 0.78 & 81.43 $\pm$ 0.95\\
		& kar-mor & 82.69 $\pm$ 2.50 & 89.01 $\pm$ 2.51 & 89.77 $\pm$ 2.60 & \textbf{92.80 $\pm$ 0.88} & 91.02 $\pm$ 1.34 & 90.89 $\pm$ 0.92\\
		& kar-pix & 81.37 $\pm$ 2.20 & 84.26 $\pm$ 2.18 & 88.06 $\pm$ 1.45 & \textbf{91.25 $\pm$ 1.40} & 90.06 $\pm$ 1.38 & 90.33 $\pm$ 1.44\\
		& kar-zer & 79.86 $\pm$ 3.52 & 86.44 $\pm$ 2.42 & 85.18 $\pm$ 1.45 & \textbf{90.09 $\pm$ 1.00} & 88.13 $\pm$ 1.78 & 88.68 $\pm$ 1.68\\
		& mor-pix & 52.88 $\pm$ 3.52 & 87.91 $\pm$ 1.46 & 90.39 $\pm$ 1.11 & \textbf{93.55 $\pm$ 1.50} & 90.15 $\pm$ 1.49 & 90.90 $\pm$ 1.33\\
		& mor-zer & 73.32 $\pm$ 1.48 & 75.38 $\pm$ 1.47 & 75.13 $\pm$ 2.35 & \textbf{77.37 $\pm$ 1.60} & 76.89 $\pm$ 1.93 & 76.49 $\pm$ 1.69\\
		& pix-zer & 55.06 $\pm$ 2.17 & 86.34 $\pm$ 1.60 & 86.35 $\pm$ 1.72 & \textbf{89.75 $\pm$ 1.73} & 88.68 $\pm$ 1.67 & 88.93 $\pm$ 1.57\\
		\hline
		\multirow{15}{*}{MFA}& fac-fou & 57.22 $\pm$ 1.87 & 89.04 $\pm$ 1.98 & 90.21 $\pm$ 1.17 & \textbf{93.70 $\pm$ 0.93} & 92.99 $\pm$ 1.21 & 93.65 $\pm$ 0.87\\
		& fac-kar & 78.14 $\pm$ 1.44 & 88.27 $\pm$ 3.13 & 81.39 $\pm$ 2.95 & 91.63 $\pm$ 1.67 & 89.46 $\pm$ 1.78 & \textbf{93.21 $\pm$ 0.91}\\
		& fac-mor & 64.00 $\pm$ 2.22 & 93.39 $\pm$ 1.00 & 83.68 $\pm$ 1.70 & \textbf{94.55 $\pm$ 0.71} & 91.83 $\pm$ 1.04 & 93.86 $\pm$ 1.37\\
		& fac-pix & 17.83 $\pm$ 2.52 & 90.55 $\pm$ 0.71 & 83.40 $\pm$ 2.22 & 91.18 $\pm$ 1.93 & 89.87 $\pm$ 1.36 & \textbf{92.57 $\pm$ 0.95}\\
		& fac-zer & 57.01 $\pm$ 3.15 & 91.37 $\pm$ 1.67 & 80.52 $\pm$ 4.19 & \textbf{92.72 $\pm$ 1.75} & 87.62 $\pm$ 1.44 & 91.97 $\pm$ 1.73\\
		& fou-kar & 59.35 $\pm$ 3.09 & 83.41 $\pm$ 3.49 & 92.13 $\pm$ 0.92 & \textbf{94.19 $\pm$ 0.98} & 91.68 $\pm$ 1.35 & 92.40 $\pm$ 1.35\\
		& fou-mor & 68.60 $\pm$ 2.16 & 76.00 $\pm$ 2.46 & 78.97 $\pm$ 2.07 & \textbf{81.02 $\pm$ 1.37} & 79.46 $\pm$ 1.67 & 79.46 $\pm$ 1.64\\
		& fou-pix & 19.81 $\pm$ 1.77 & 84.22 $\pm$ 3.32 & 89.34 $\pm$ 1.56 & \textbf{94.54 $\pm$ 0.75} & 91.90 $\pm$ 1.25 & 92.79 $\pm$ 1.08\\
		& fou-zer & 64.45 $\pm$ 3.96 & 76.22 $\pm$ 1.93 & 77.24 $\pm$ 1.55 & \textbf{81.75 $\pm$ 0.99} & 81.05 $\pm$ 0.74 & 81.25 $\pm$ 0.82\\
		& kar-mor & 72.25 $\pm$ 1.84 & 87.86 $\pm$ 2.22 & 89.96 $\pm$ 2.21 & \textbf{93.14 $\pm$ 1.15} & 92.39 $\pm$ 1.70 & 92.36 $\pm$ 1.25\\
		& kar-pix & 23.17 $\pm$ 2.91 & 82.51 $\pm$ 2.56 & 82.02 $\pm$ 3.09 & \textbf{91.04 $\pm$ 1.40} & 88.38 $\pm$ 1.83 & 90.57 $\pm$ 1.21\\
		& kar-zer & 62.52 $\pm$ 3.60 & 83.82 $\pm$ 2.42 & 85.00 $\pm$ 1.59 & 90.23 $\pm$ 1.48 & 87.98 $\pm$ 1.70 & \textbf{90.29 $\pm$ 1.12}\\
		& mor-pix & 48.37 $\pm$ 6.81 & 91.44 $\pm$ 0.96 & 85.82 $\pm$ 2.68 & \textbf{93.21 $\pm$ 1.09} & 91.37 $\pm$ 1.79 & 93.08 $\pm$ 1.21\\
		& mor-zer & 69.35 $\pm$ 2.61 & 76.57 $\pm$ 2.05 & 74.28 $\pm$ 2.29 & 77.30 $\pm$ 1.55 & 76.96 $\pm$ 2.65 & \textbf{77.97 $\pm$ 2.26}\\
		& pix-zer & 27.19 $\pm$ 4.49 & 86.54 $\pm$ 1.54 & 81.51 $\pm$ 1.94 & 90.25 $\pm$ 1.74 & 88.45 $\pm$ 1.90 & \textbf{90.54 $\pm$ 1.03}\\
		\hline
	\end{tabular}
	\end{tiny}
	\vspace{-0.2in}
\end{table*}

\begin{figure*}
	\begin{tabular}{@{}c@{}c@{}c@{}c@{}c@{}}
		\includegraphics[width=0.2\textwidth]{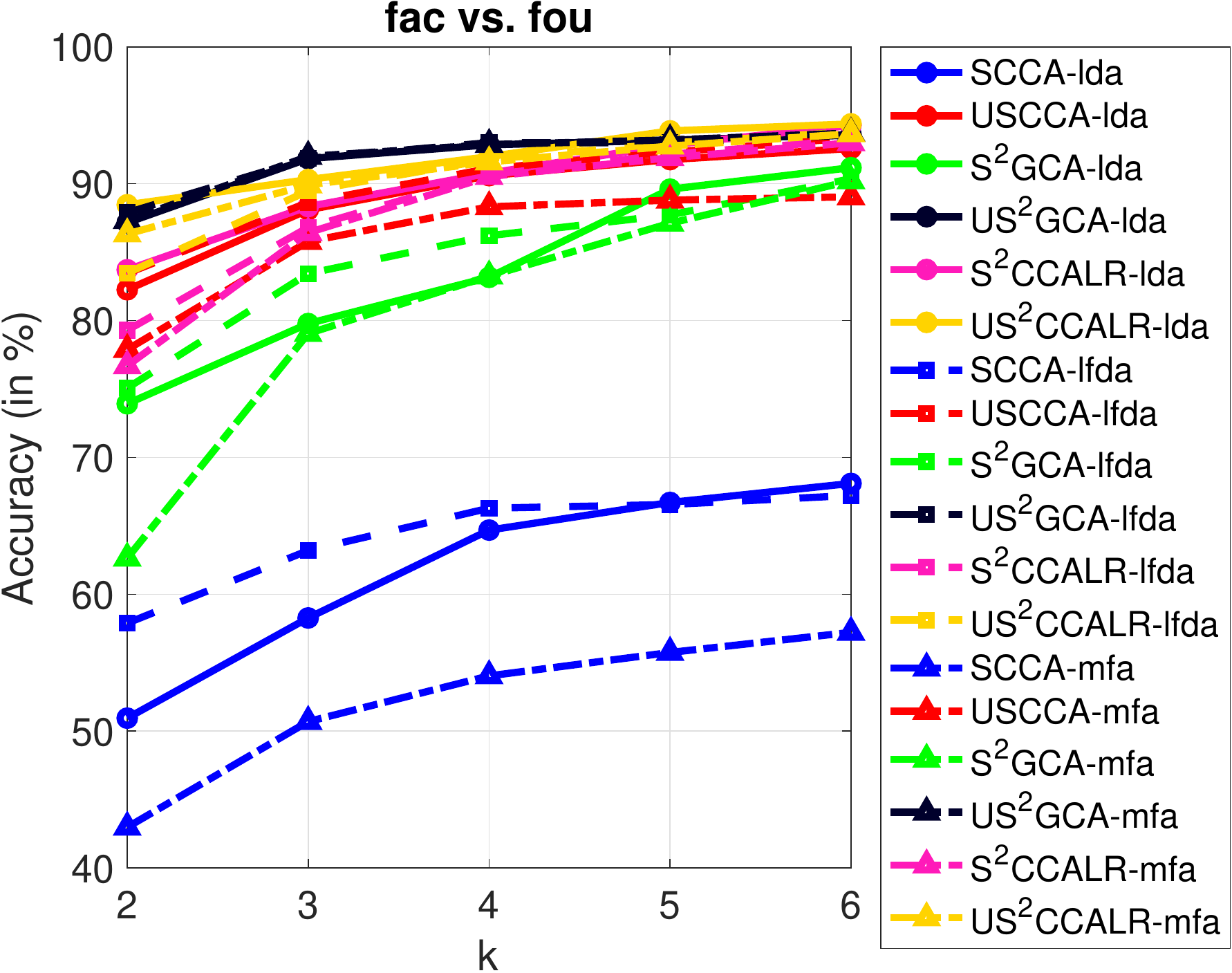}&
		\includegraphics[width=0.2\textwidth]{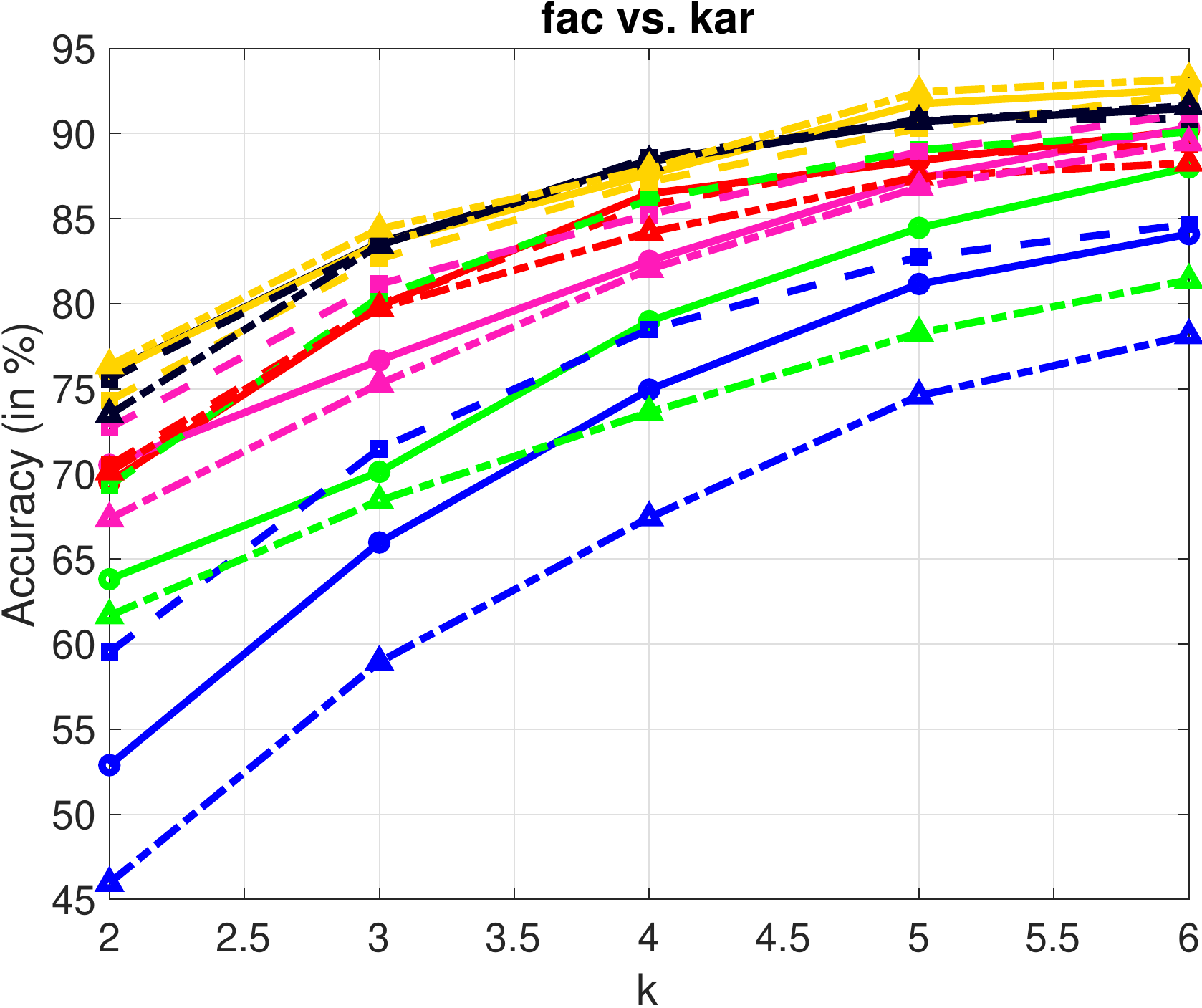}&
		\includegraphics[width=0.2\textwidth]{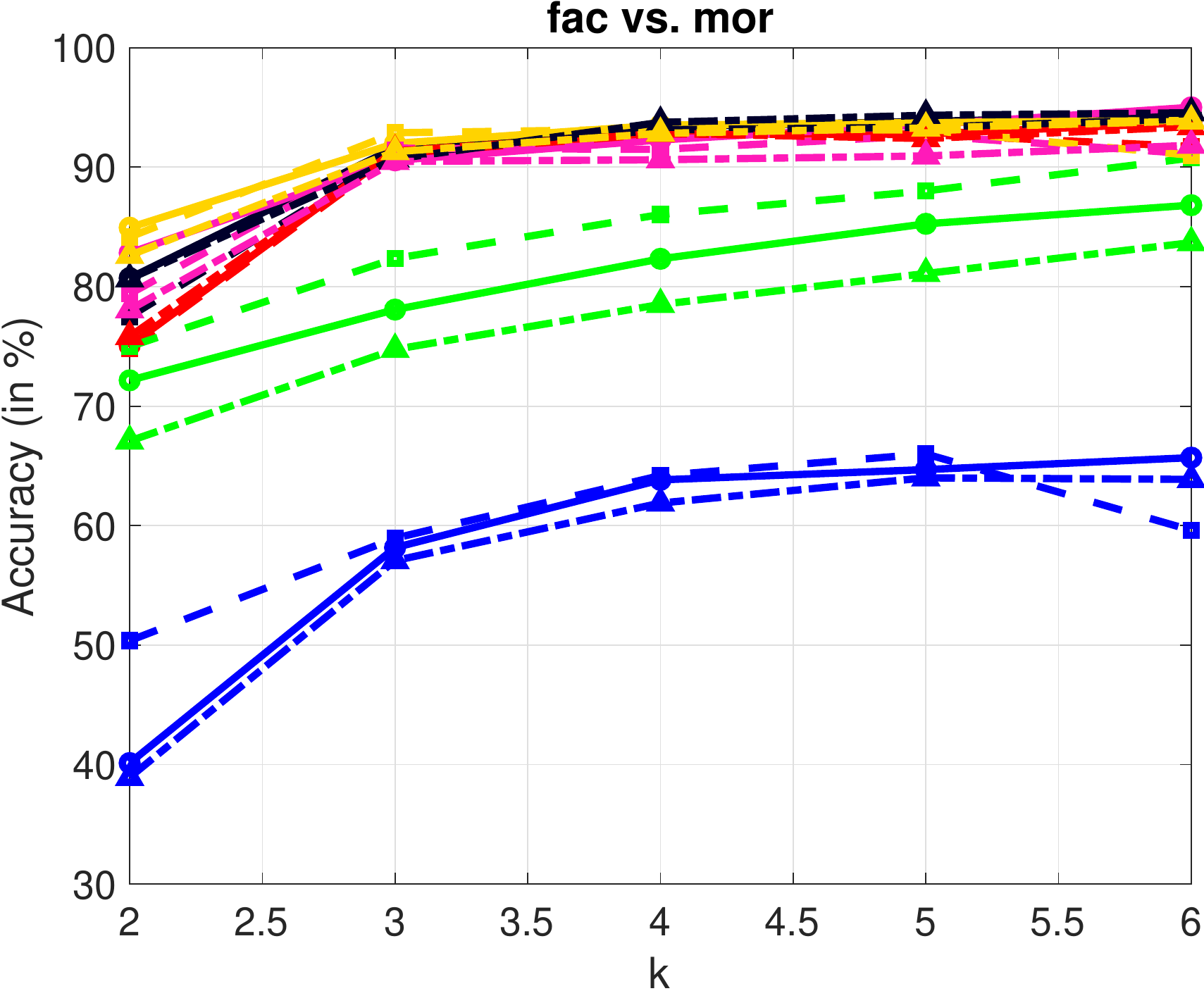}&
		\includegraphics[width=0.2\textwidth]{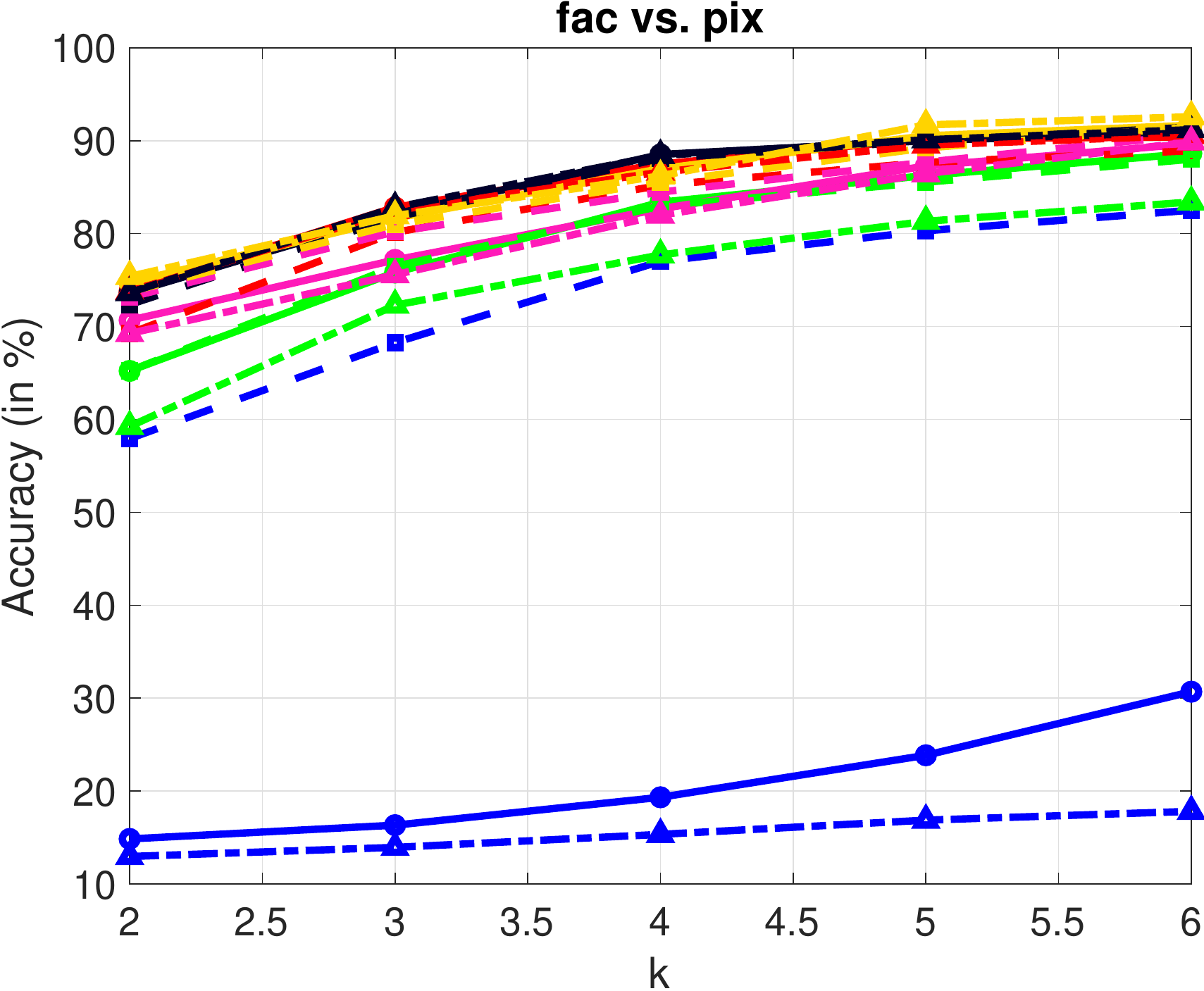} &
		\includegraphics[width=0.2\textwidth]{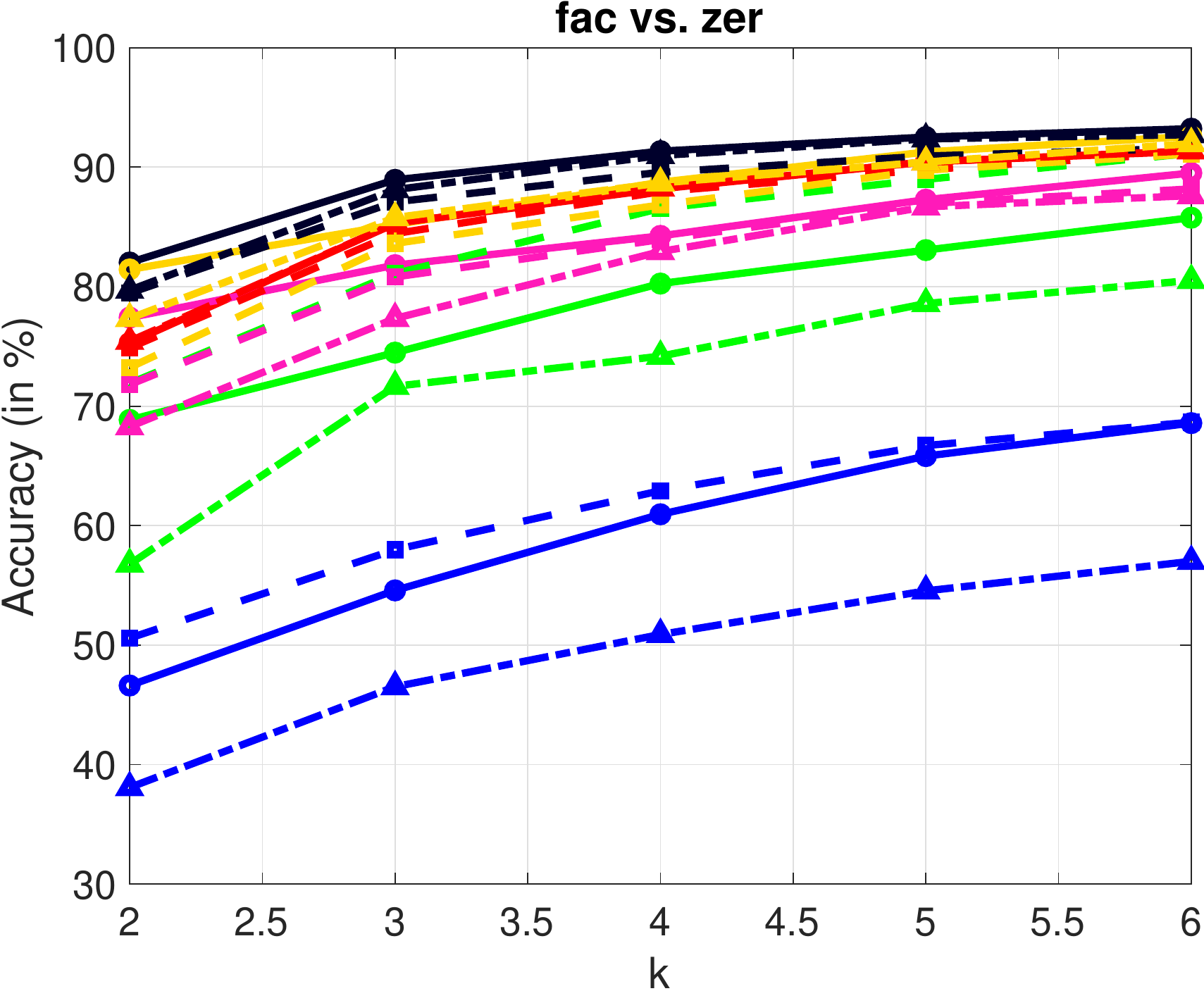}\\
		\includegraphics[width=0.2\textwidth]{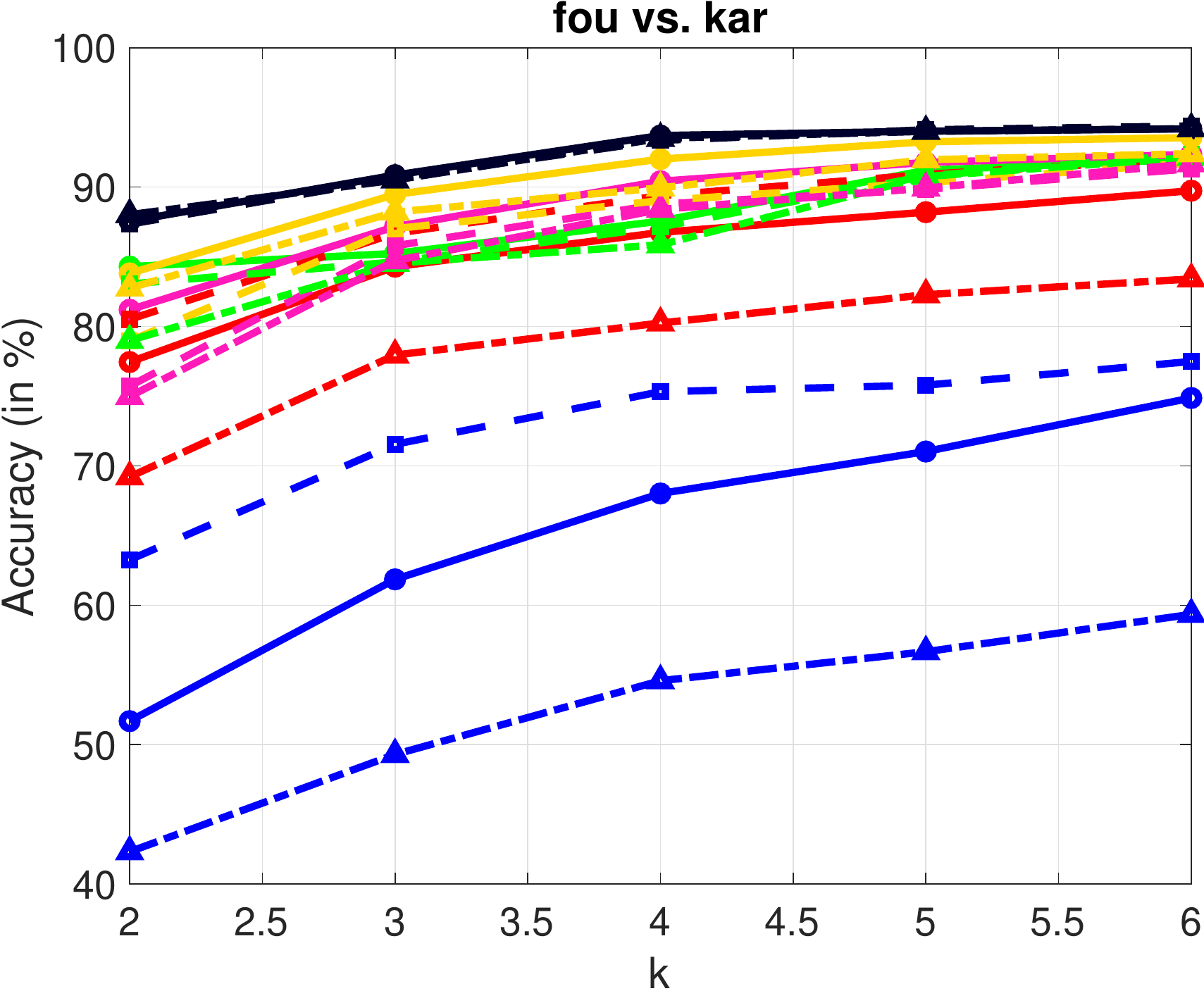}&
		\includegraphics[width=0.2\textwidth]{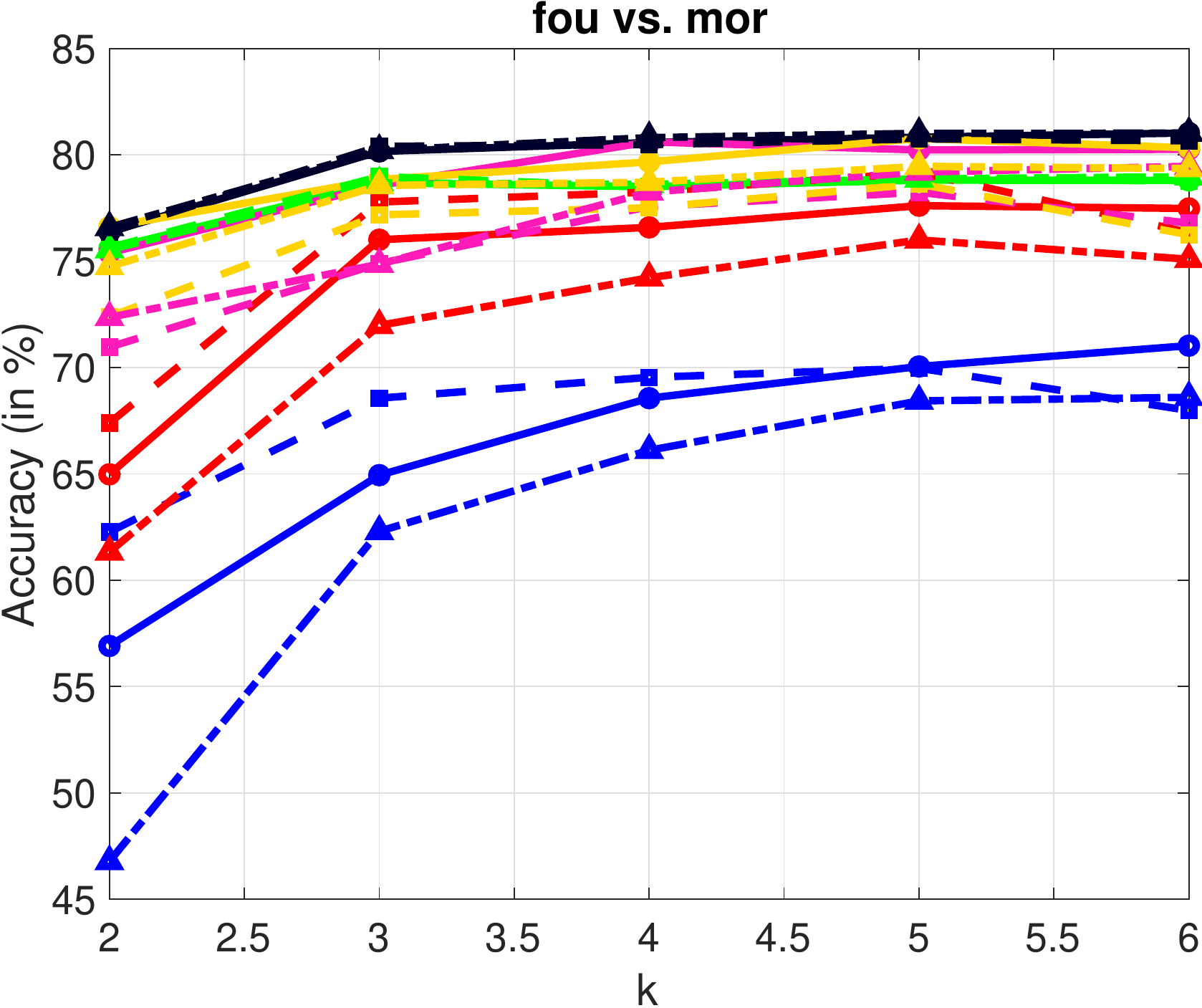}&
		\includegraphics[width=0.2\textwidth]{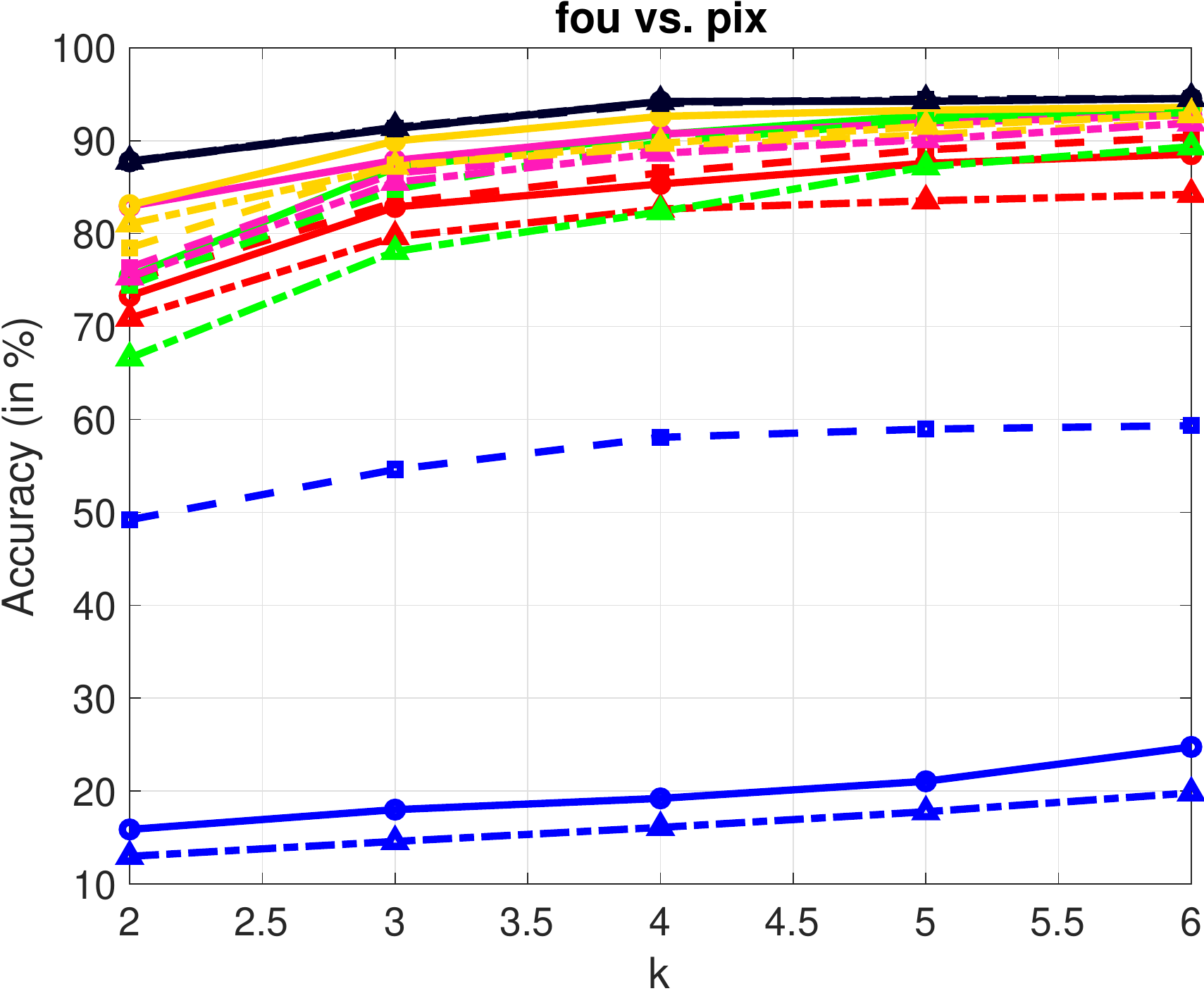}&
		\includegraphics[width=0.2\textwidth]{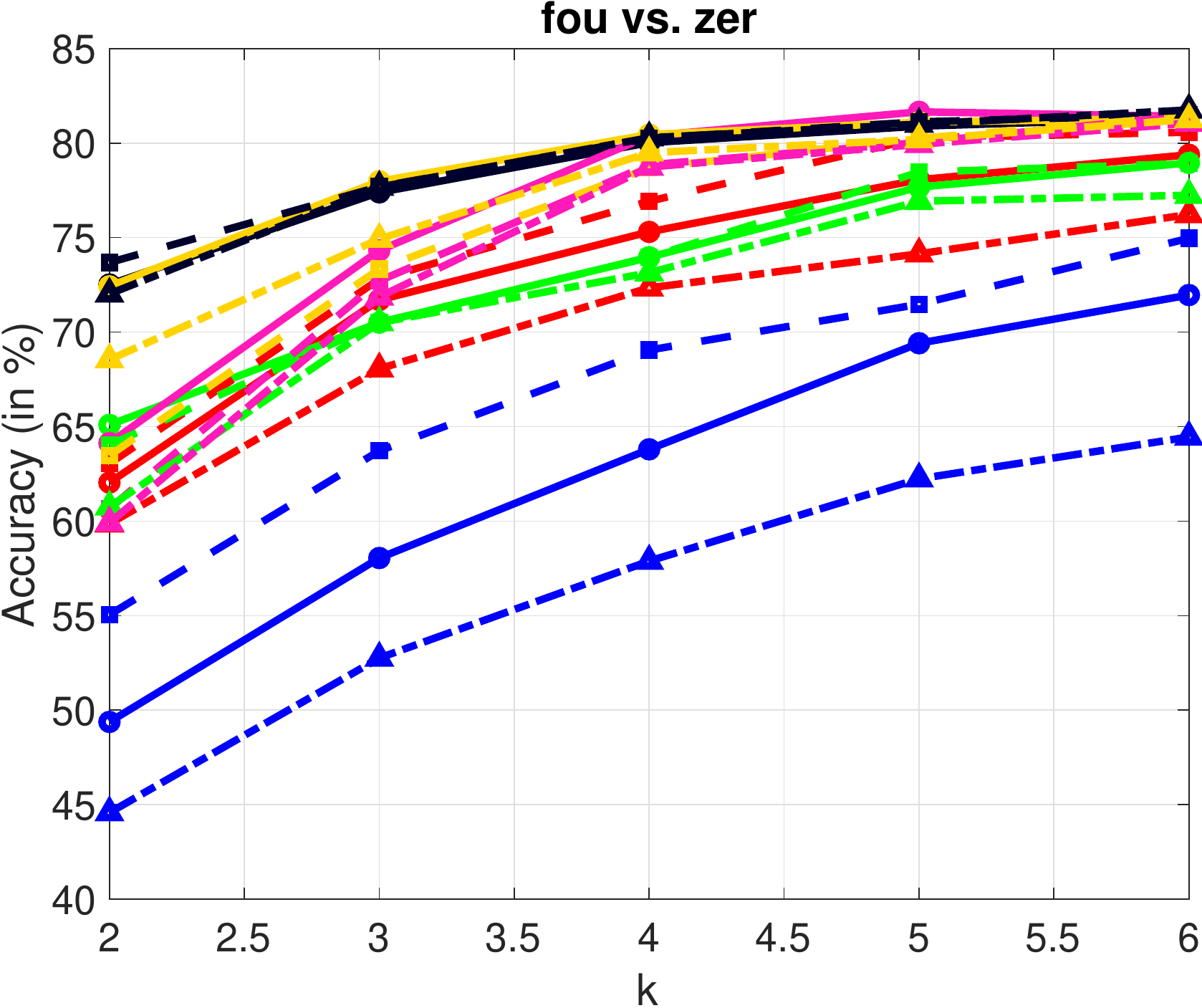}&
		\includegraphics[width=0.2\textwidth]{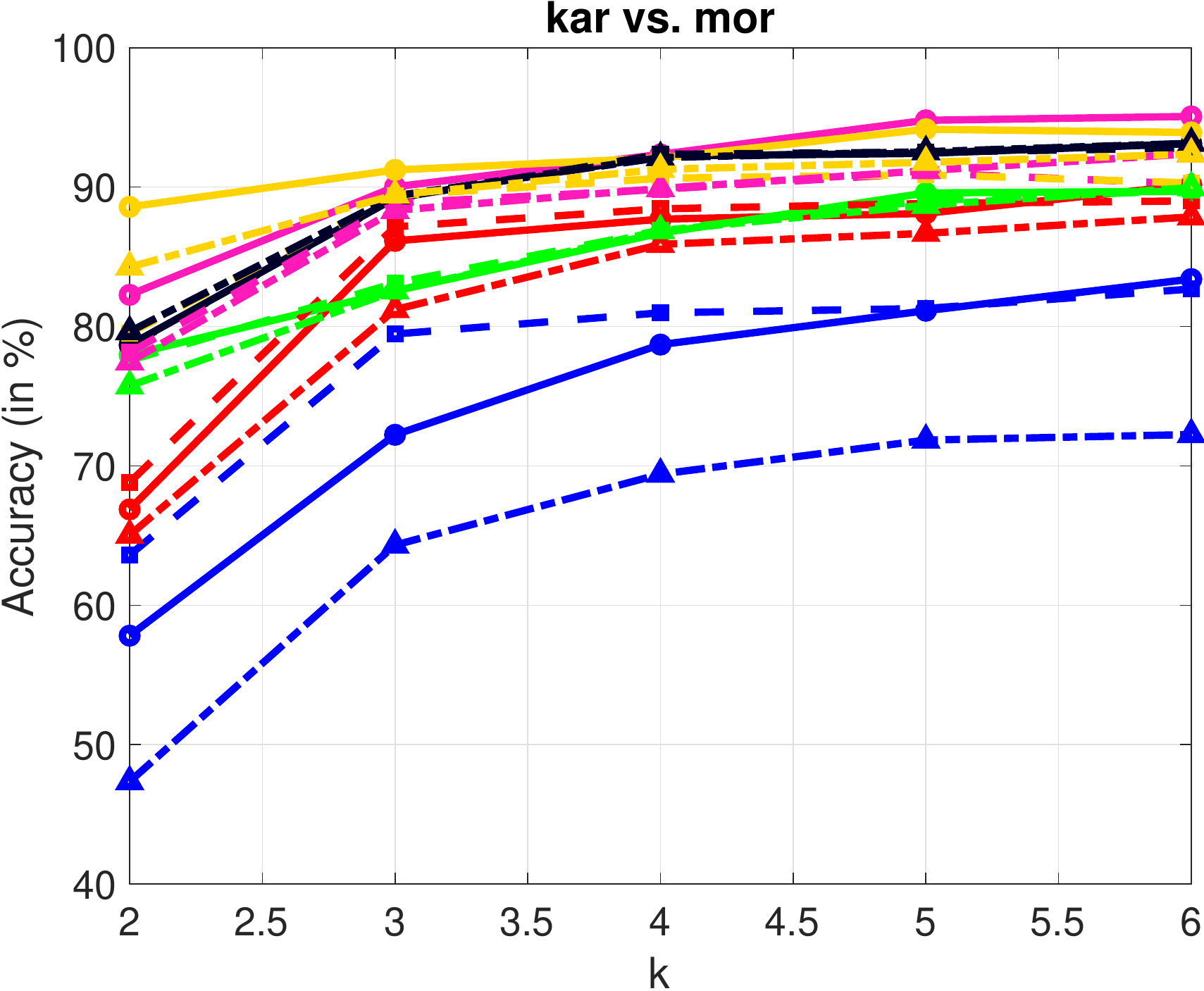}\\
		\includegraphics[width=0.2\textwidth]{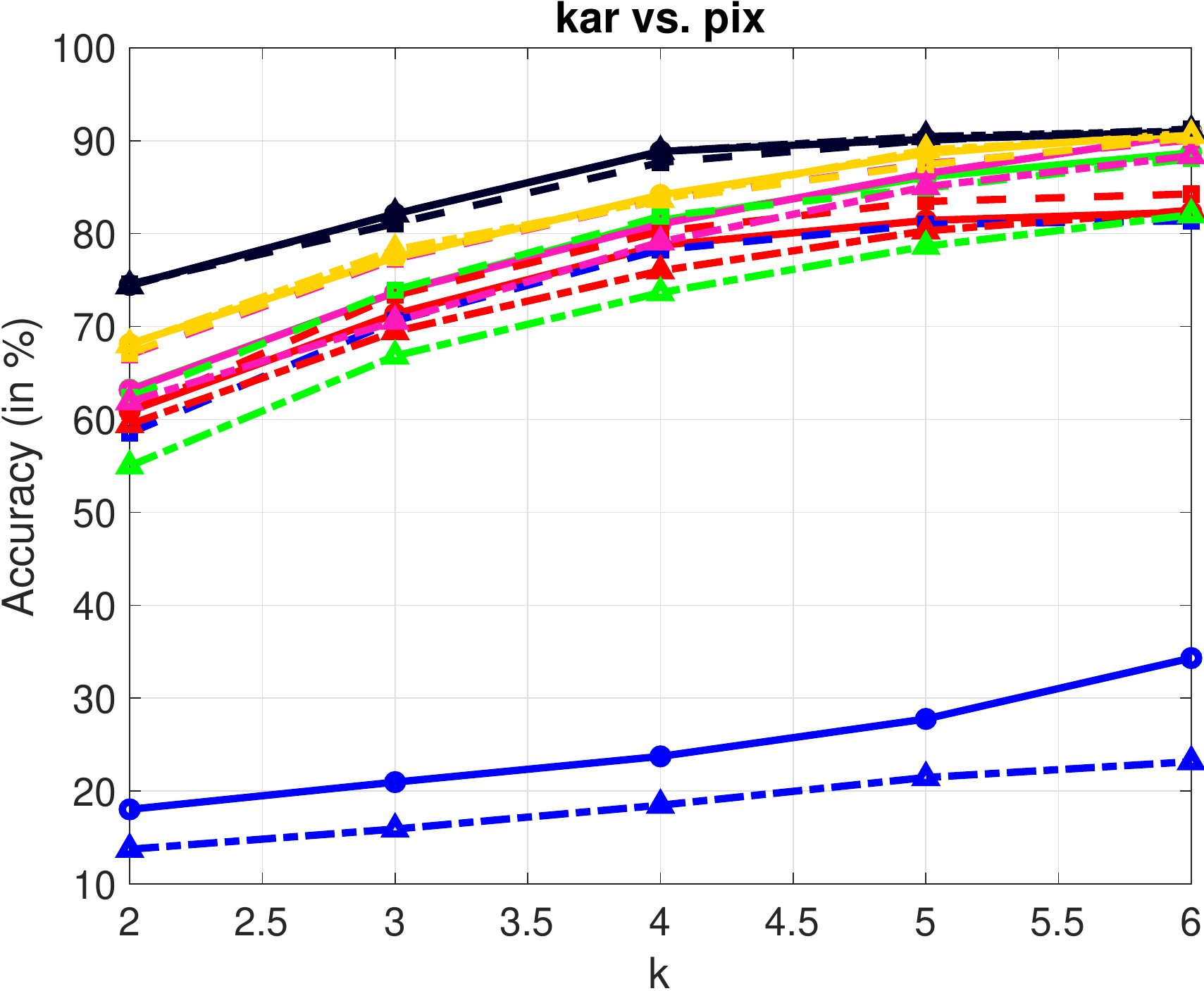}&
		\includegraphics[width=0.2\textwidth]{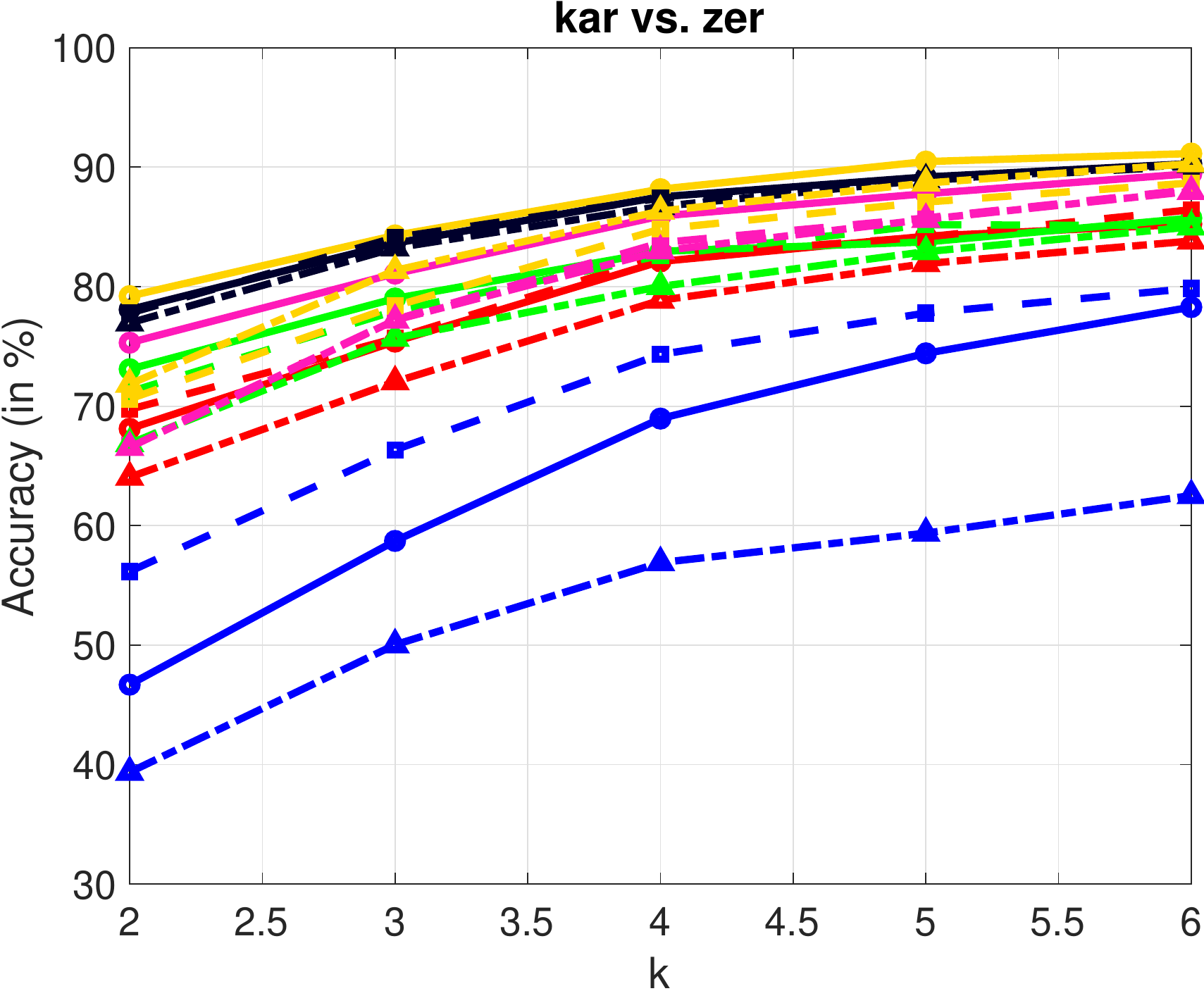}&
		\includegraphics[width=0.2\textwidth]{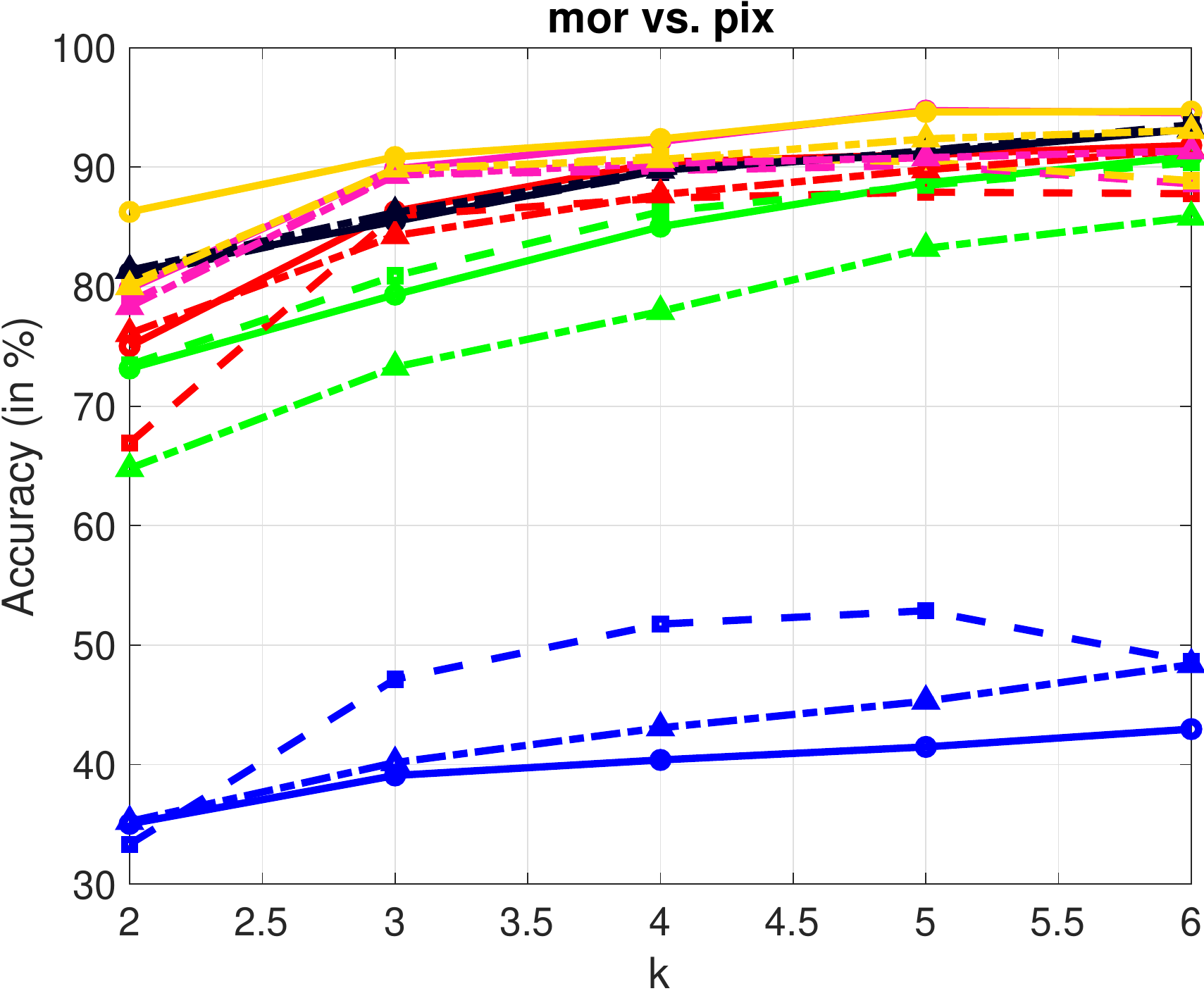}&
		\includegraphics[width=0.2\textwidth]{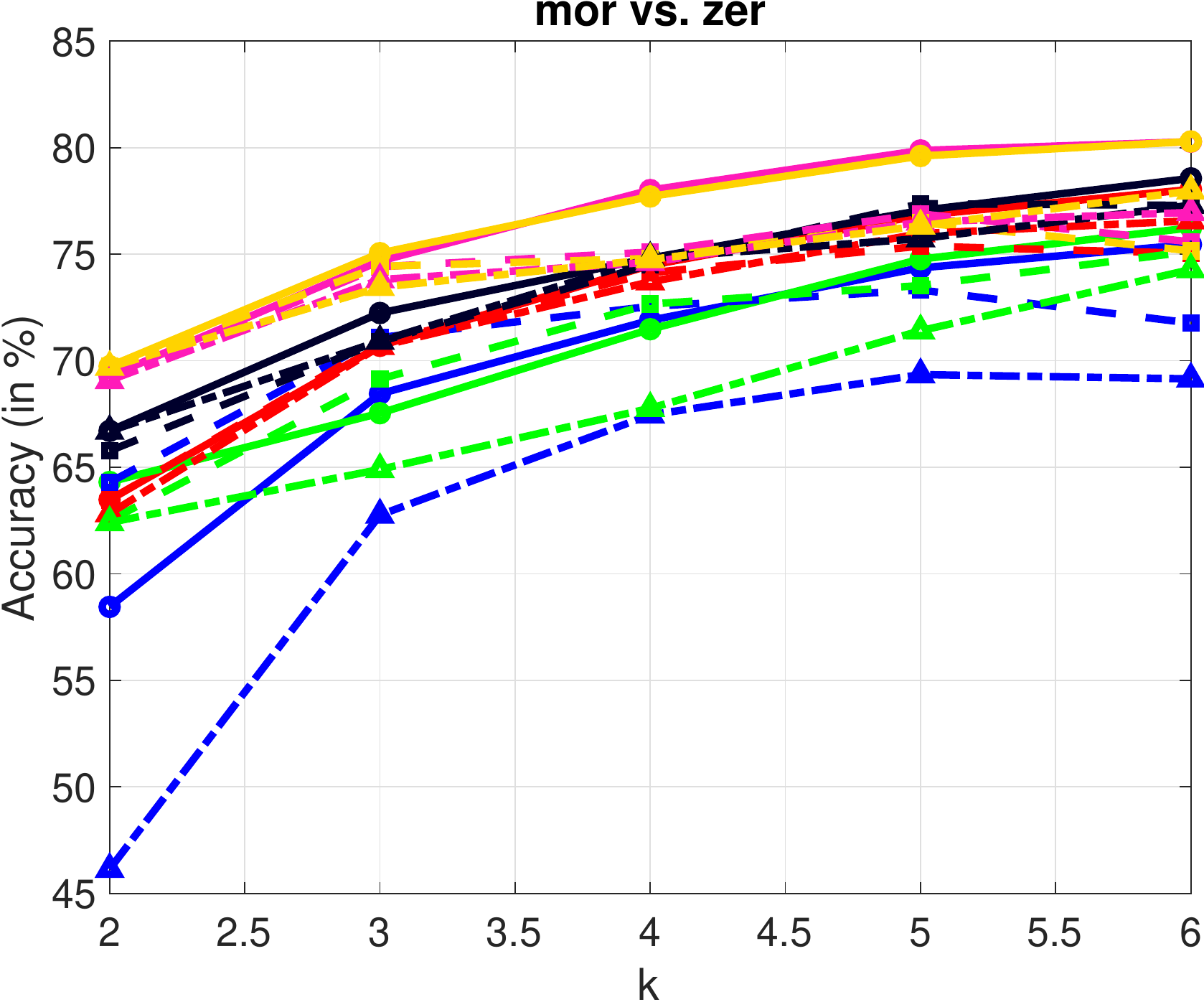}&
		\includegraphics[width=0.2\textwidth]{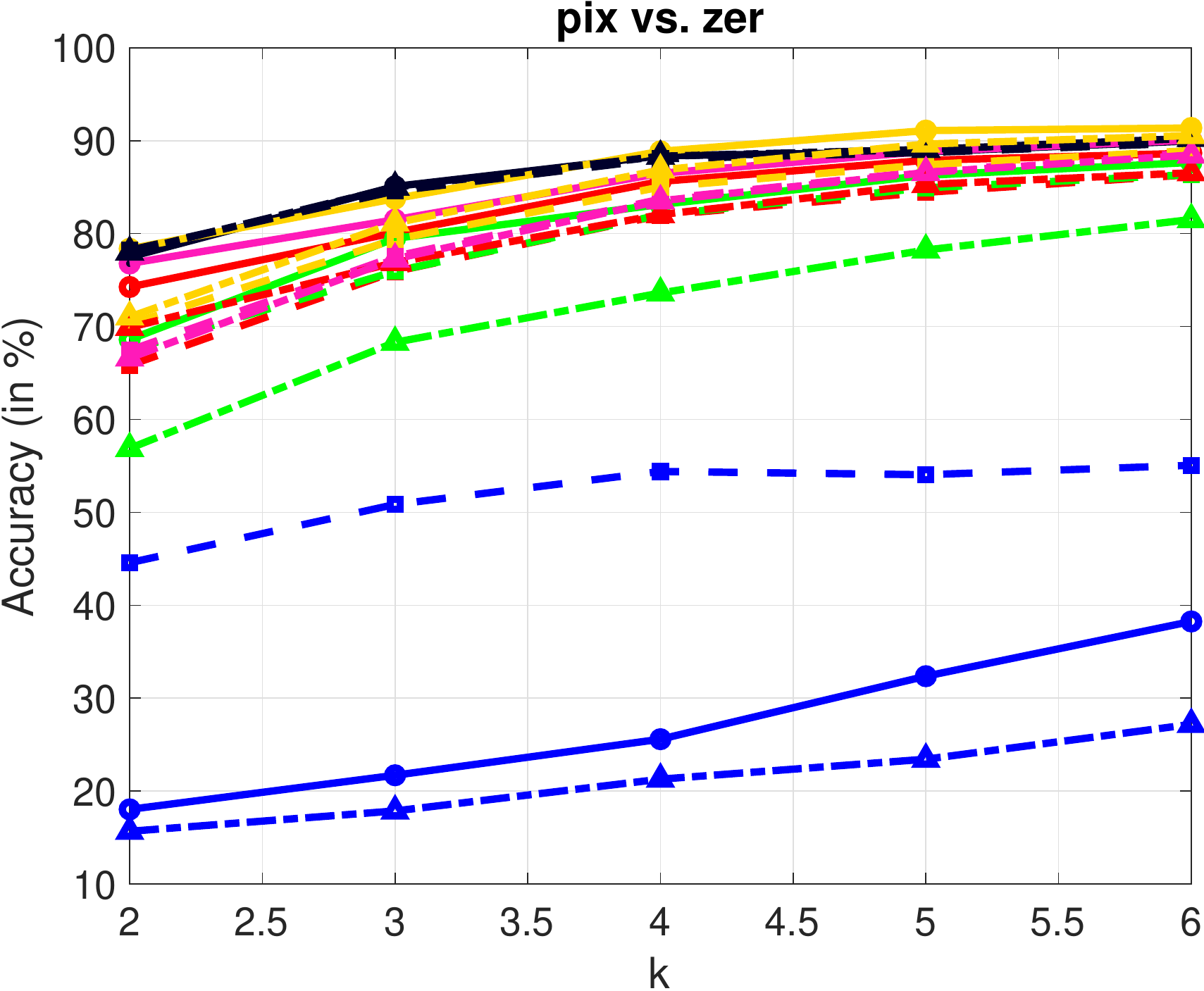}			
	\end{tabular} \vspace{-0.1in}
	\caption{Average accuracy by compared methods  over $10$ random splits on $15$ datasets generated on mfeat as $k$ varies from 2 to 6.} \label{fig:supervised-mfeat}
	\vspace{-0.1in}
\end{figure*}

\subsubsection{Unsupervised semi-paired subspace learning}
We compare our proposed models against three baseline methods for  unsupervised semi-paired subspace learning:
\begin{itemize}
	\item CCA \cite{hardoon2004canonical}. To make sure that the involved GEP doesn't encounter any numerical singularity issue,
	we add a small diagonal matrix, e.g., $10^{-6} I_{d_s}$, to each covariance matrix. The same is applied to all other baseline
	methods;
	
	\item SemiCCA \cite{kimura2013semicca}. The tradeoff-parameter $\gamma$ between CCA and the variant of PCA are
	tuned in the range $\gamma \in [0.01,0.05,0.1,0.5,0.9,0.95,0.99]$;
	
	\item SemiCCALR \cite{belkin2006manifold}. Graph Laplacian matrices are constructed as the heat kernel matrices
	over a $k$-nearest neighbor graph of each view. The bandwidth of the heat kernel is set to a scaled value of the mean distances over all paired samples, where the scale is tuned in $[1/4,1/2,1,2,4]$.
	The regularization parameter $\gamma_2$ is tuned in the range $\gamma_2\in[10^{-3},10^{-2}, 10^{-1},1, 10,100,1000]$ and $\gamma_1 = 0$ since $10^{-6} I_{d_s}$ added to the covariance matrix can be considered as a regularization;
	
	\item USemiCCA: the proposed model (\ref{op:usemicca}), in analogy to SemiCCA. The same parameter setting as for SemiCCA is used;
	
	\item USemiCCALR:  the proposed model (\ref{op:usemicca-lr}), in analogy to SemiCCALR. The same parameter setting
	as for SemiCCALR is used.
\end{itemize}
In addition to the hyper-parameters of each compared model, the dimension $k$ of the common latent space is also important. As the view mor only consists of $6$ features, we restrict $k \in [2, 6]$.
The best results over all tuned parameters are reported in terms of the average accuracy over $10$ random splits.

Experimental results on mfeat are shown in Table \ref{tab:mfeat-unsup}. We have the following observations: 1) all four models leveraging unpaired data can significantly outperform CCA; 
2) USemiCCA outperforms SemiCCA, which shows the advantage of uncorrelated features for classification; 3) SemiCCALR and USemiCCALR perform equally well, which indicates that our proposed optimization algorithm can reach an approximation solution of similar quality to the GEP solver as the two optimization problems are equivalent as shown in Section~\ref{sec:cca}.

\subsubsection{Semi-supervised semi-paired subspace learning} \label{sec:sssl-mfeat}
In this section, we explore  semi-paired subspace learning in the semi-supervised setting.
We compare our proposed models against three baseline methods for semi-supervised semi-paired subspace learning:
\begin{itemize}
	\item SCCA \eqref{op:scca}: CCA with a supervised regularizer,
	\item S$^2$GCA (\ref{op:s2gca-primal}): \cite{chen2012unified},
	\item S$^2$CCALR \eqref{op:sscca-lr}: SemiCCALR with a supervised regularizer,
	\item USCCA: the proposed model (\ref{op:ucca}),	
	\item US$^2$GCA: the proposed model (\ref{op:us2gca}),
	\item US$^2$CCALR: the proposed model (\ref{op:usemi2cca-lr}).
\end{itemize}
In Section~\ref{sec:ss-spcca}, we have introduced three different graph constructers, based on LDA, LFDA, and MFA, respectively, to produce supervised regularizers. We add  suffixes ``-lda'', ``-lfda'', ``-mfa'' to each baseline method
to specify the corresponding graph construction method. For example, US$^2$GCA-mfa stands for  US$^2$GCA with MFA. In addition to the hyper-parameters in their unsupervised counterparts, our semi-supervised models have additional hyper-parameters including $\eta$ and graph construction parameters. We tune $\eta \in \{10^{-3},10^{-2}, 10^{-1},1, 10,100,1000\}$. The number of neighbors in LFDA is tuned in $\{3, 5, 7, 10, 20\}$ and $k_1, k_2$ in MFA are tuned in $\{3, 5, 7, 10, 20\}$.
Finally, NNC is trained on the projected data points of the labeled data.

The average accuracy with standard deviation by all compared methods on the $15$ pairs of views of mfeat over $10$ random training/testing splits are shown in Table \ref{tab:sssl-mfeat}. We observe that 1) SCCA performs the worst among semi-paired methods, 2) semi-paired semi-supervised methods outperform semi-paired unsupervised methods, and 3) methods with uncorrelated constraints beat their counterparts. In Fig.~\ref{fig:supervised-mfeat}, we show the average accuracy by compared methods as $k$ varies in $[2,6]$ on mfeat.    It is observed that 1) all methods demonstrate the same trend of improved classification performance as $k$ increases, 2) US$^2$GCA and US$^2$CCALR generally outperform the others for any given $k$. These observations imply that learning uncorrelated features can improve the classification performance of existing models in semi-supervised semi-paired subspace learning.

\begin{figure}
	\centering
	\begin{tabular}{@{}c@{}c@{}}
		\includegraphics[width=0.45\textwidth]{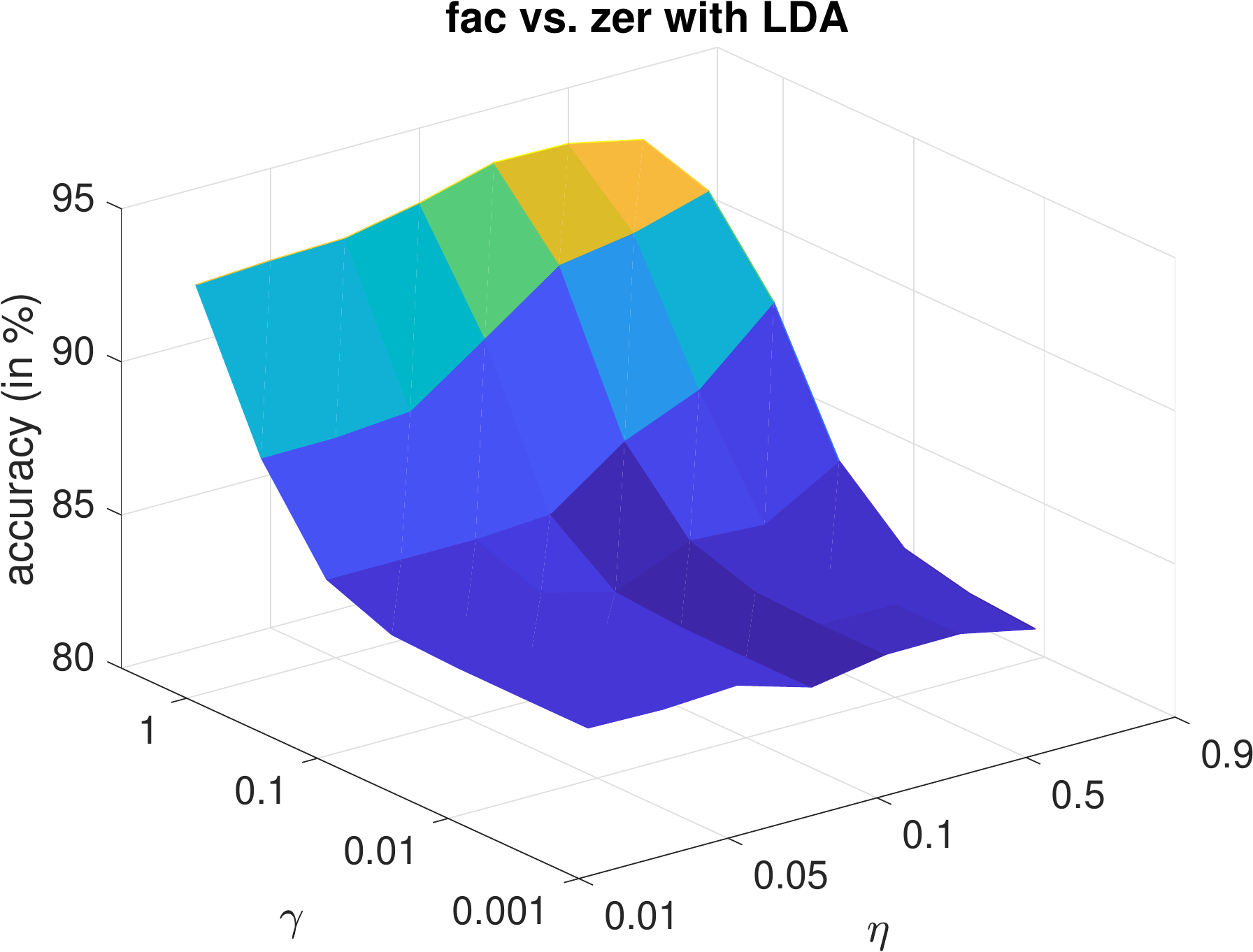} &
		\includegraphics[width=0.45\textwidth]{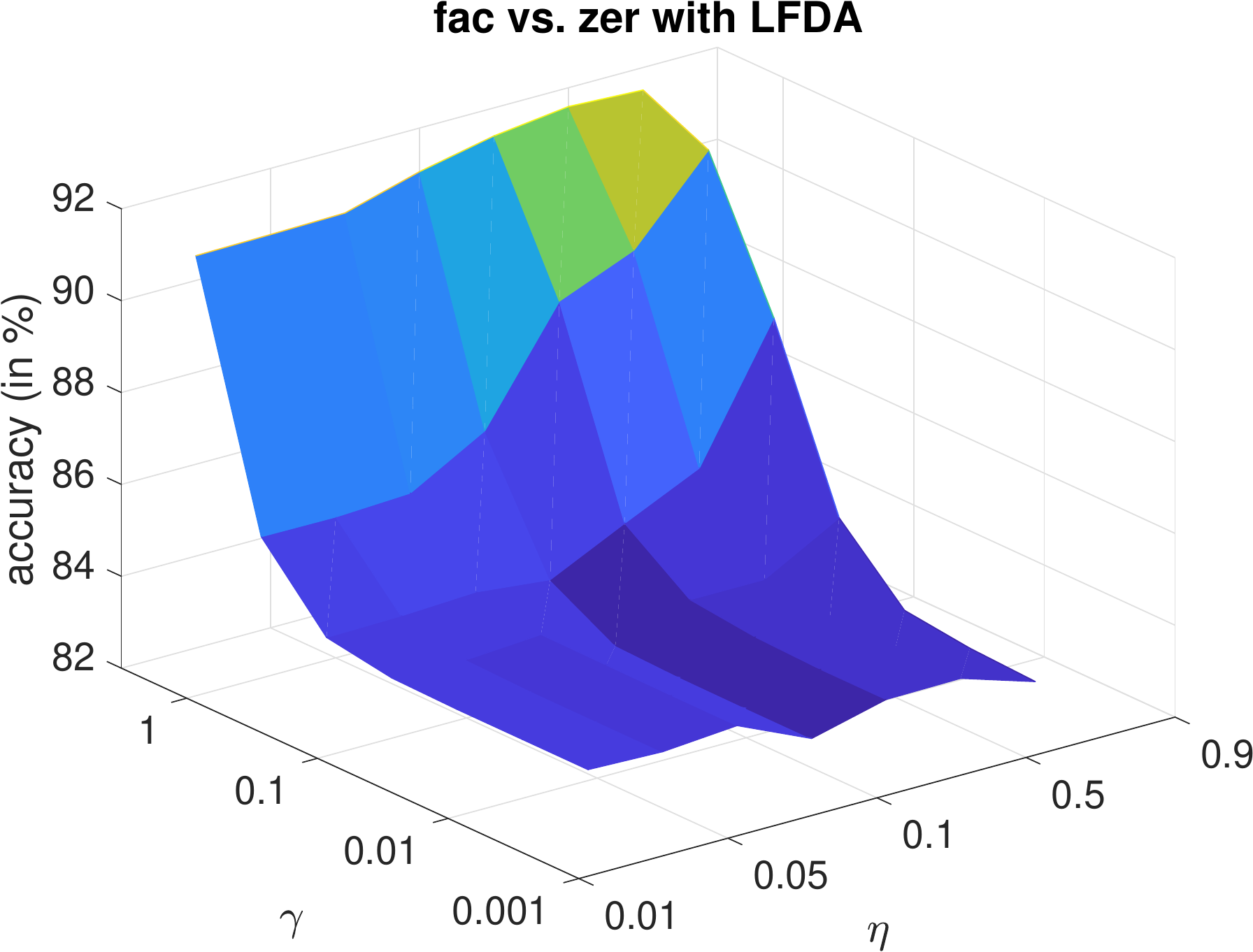}	\\
		(a) & (b) \\
		\includegraphics[width=0.45\textwidth]{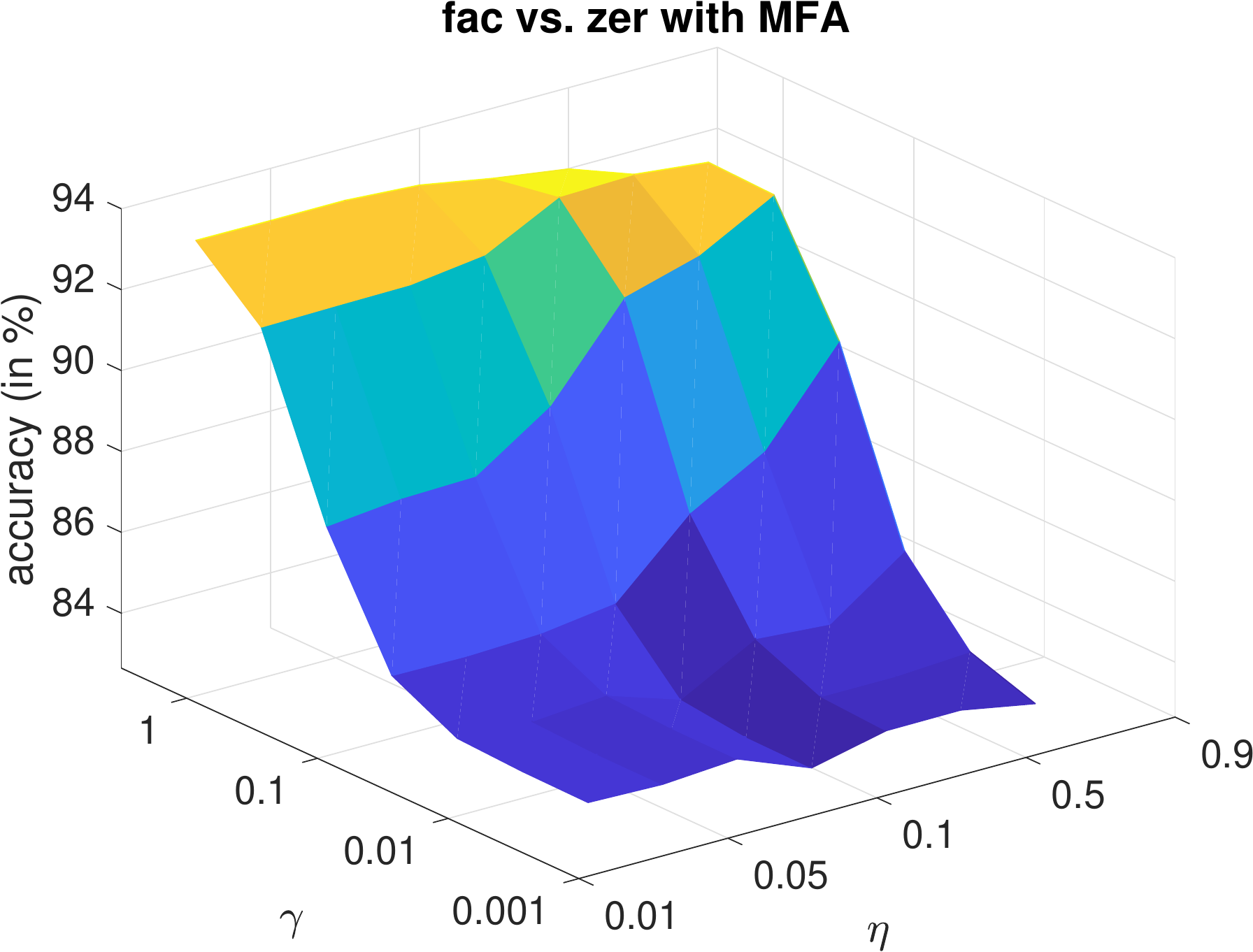} &
		\includegraphics[width=0.43\textwidth]{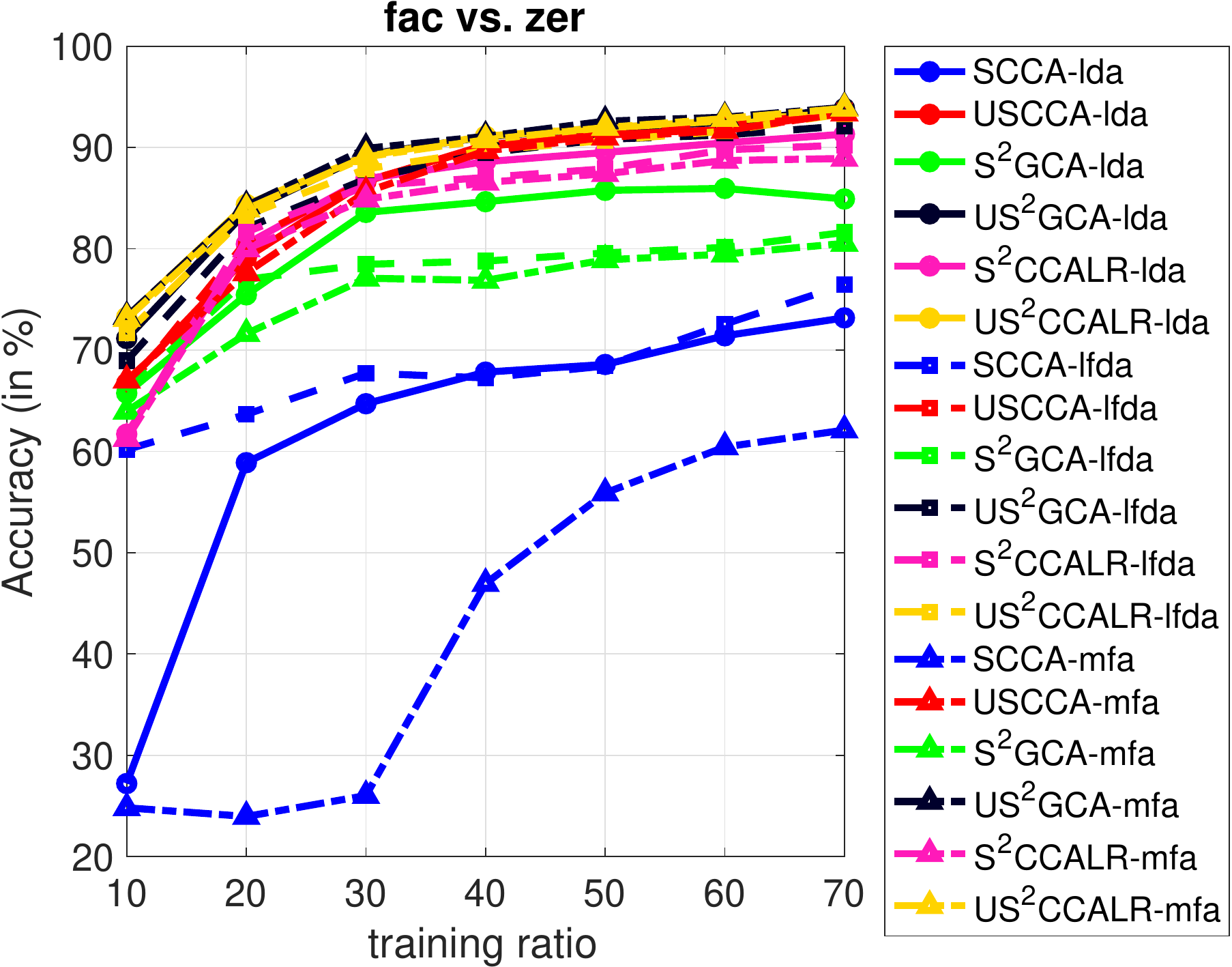}\\
		(c) & (d)
	\end{tabular}	\vspace{-0.1in}
	\caption{Sensitivity analysis. (a)-(c): accuracy by US$^2$GCA on two views fac and zer of mfeat for varying $\gamma$ and $\eta$ and for using LDA, LFDA and MFA to incorporate supervised information; (d): accuracy by compared methods for varying  training ratio from 10\% to 70\%.}	 \label{fig:sensitivity}
	\vspace{-0.15in}
\end{figure}

\begin{landscape}
\begin{table*}
	\centering
	\caption{Average accuracy with standard deviation by six methods on Pascal  data over 10 randomly drawn training and testing splits with three different scatter construction approaches. The best results are in bold.}\label{tab:pascal}
	\vspace{-0.1in}
	\begin{small}
	\begin{tabular}{@{}c|c|c|cccccc@{}}
		\hline
		tested view & graph construction	&view 1 - view 2& SCCA& USCCA& S$^2$GCA& US$^2$GCA& S$^2$CCALR& US$^2$CCALR\\
		\hline
		\multirow{3}{*}{image+text} & \multirow{1}{*}{LDA}& image-text & 10.48 $\pm$ 1.11 & 29.10 $\pm$ 2.58 & 32.62 $\pm$ 3.52 & \textbf{37.32 $\pm$ 2.98} & 26.94 $\pm$ 2.71 & 26.36 $\pm$ 3.33\\\cline{2-9}
		& \multirow{1}{*}{LFDA}& image-text & 23.02 $\pm$ 4.00 & 30.62 $\pm$ 3.77 & 31.92 $\pm$ 3.68 & \textbf{33.94 $\pm$ 3.91} & 25.44 $\pm$ 2.36 & 26.58 $\pm$ 3.53\\
		\cline{2-9}
		& \multirow{1}{*}{MFA}& image-text & 11.66 $\pm$ 2.80 & 31.48 $\pm$ 3.05 & 32.90 $\pm$ 4.11 & \textbf{35.70 $\pm$ 4.18} & 25.98 $\pm$ 3.32 & 26.70 $\pm$ 2.36\\
		\cline{2-9}
		
		\hline
		\multirow{3}{*}{image} &\multirow{1}{*}{LDA}& image-text & 6.82 $\pm$ 0.82 & 11.58 $\pm$ 1.25 & 11.02 $\pm$ 1.38 & \textbf{11.62 $\pm$ 1.66} & 8.34 $\pm$ 1.78 & 9.38 $\pm$ 1.91\\
		\cline{2-9}
		&\multirow{1}{*}{LFDA}& image-text & 10.72 $\pm$ 2.04 & \textbf{11.34 $\pm$ 1.28} & 10.98 $\pm$ 1.44 & 11.12 $\pm$ 1.40 & 9.24 $\pm$ 1.96 & 9.08 $\pm$ 1.74\\
		\cline{2-9}
		&\multirow{1}{*}{MFA}& image-text & 6.68 $\pm$ 1.33 & \textbf{11.68 $\pm$ 1.71} & 11.26 $\pm$ 1.24 & 11.14 $\pm$ 1.10 & 8.12 $\pm$ 1.67 & 9.04 $\pm$ 1.27\\
		\hline		
		\multirow{3}{*}{Text} &\multirow{1}{*}{LDA}& image-text & 18.42 $\pm$ 2.40 & 41.64 $\pm$ 3.43 & 42.60 $\pm$ 4.00 & \textbf{48.48 $\pm$ 3.11} & 43.92 $\pm$ 3.97 & 44.10 $\pm$ 3.95\\\cline{2-9}
		& \multirow{1}{*}{LFDA}& image-text & 29.96 $\pm$ 3.17 & 40.80 $\pm$ 3.52 & 42.26 $\pm$ 3.82 & \textbf{46.32 $\pm$ 3.74} & 43.72 $\pm$ 3.86 & 43.20 $\pm$ 4.24\\\cline{2-9}
		& \multirow{1}{*}{MFA}& image-text & 17.14 $\pm$ 3.70 & 42.86 $\pm$ 4.76 & 42.14 $\pm$ 3.76 & \textbf{44.92 $\pm$ 3.66} & 41.20 $\pm$ 3.35 & 44.42 $\pm$ 3.16\\
		\hline
	\end{tabular}
	\end{small}
\end{table*}
\end{landscape}

\subsubsection{Sensitivity analysis}
We conduct the sensitivity analysis of our proposed methods from two perspectives: the hyper-parameters and the training ratios. To save space, we only show the results by US$^2$GCA  which uses two parameters $\gamma$ and $\eta$ to control the importance of unpaired data and the  labeled data on $10$ random training/testing splits.
Fig.~\ref{fig:sensitivity} displays the variation of the best classification accuracy over $k \in [2, 6]$ as $\gamma$ and $\eta$ vary on  two views fac and zer. Specifically, Fig.~\ref{fig:sensitivity}(a)-(c) show the accuracy changes for three different scatter matrices, respectively. It is observed that larger $\gamma$ and $\eta$ can generally achieve better results. This implies that unpaired data and a small amount of supervised data do improve  classification accuracy. In Fig.~\ref{fig:sensitivity}(d), it is demonstrated that our proposed methods consistently perform better when more training data become available. Hence, our methods are robust to varying training data.

\begin{figure*}
	\centering
	\begin{tabular}{ccc}
		\includegraphics[width=0.31\textwidth]{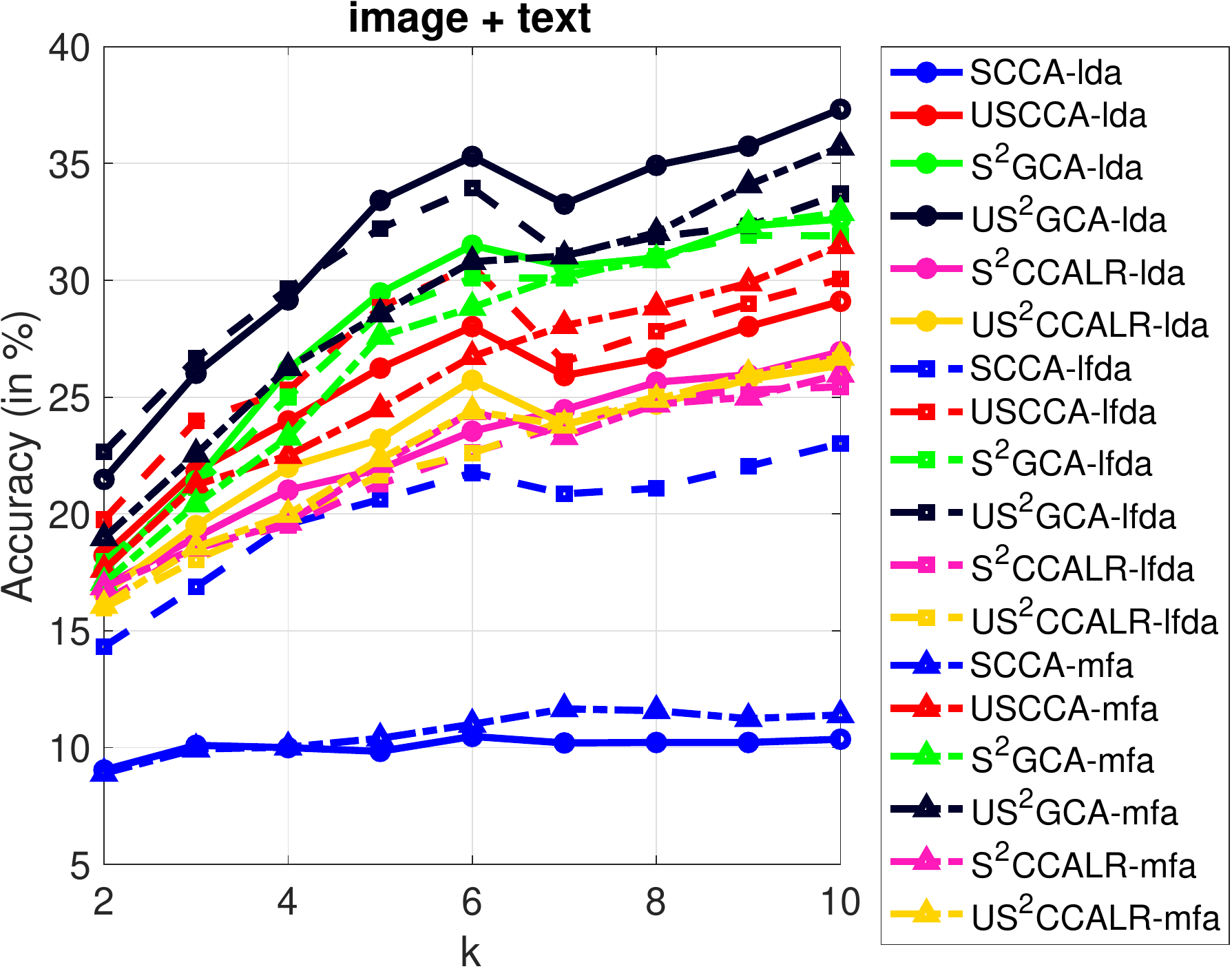} &
		\includegraphics[width=0.31\textwidth]{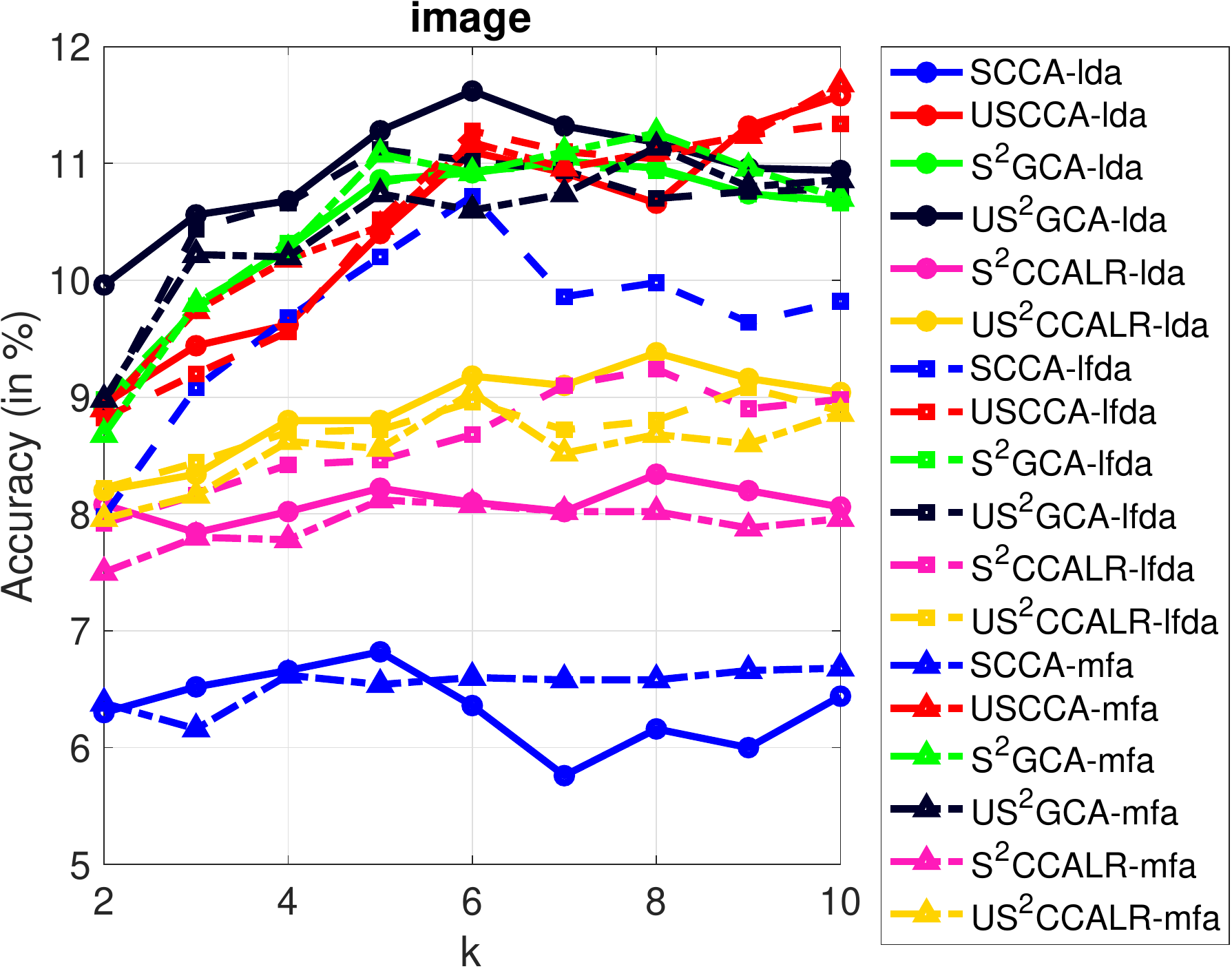} &
		\includegraphics[width=0.31\textwidth]{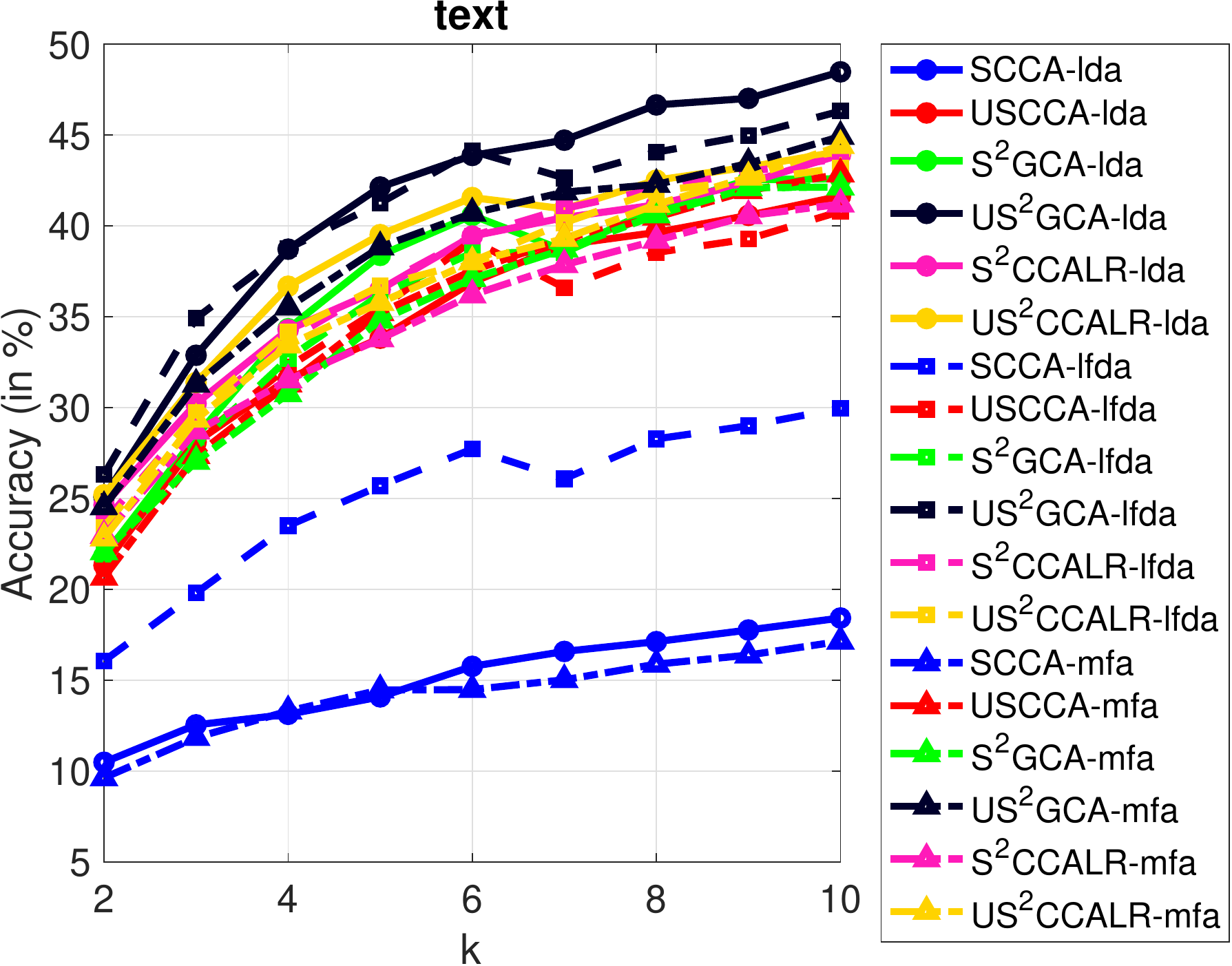}
	\end{tabular}	\vspace{-0.1in}
	\caption{Average accuracy by compared methods for varying $k \in [2, 10]$ over $10$ random splits on Pascal data.}\label{fig:Pascal}	
\end{figure*}

\subsection{Multi-modal data}
In this section, we consider  multi-modal data,
Pascal dataset, which contains 1000 pairs of image and text from 20 categories \cite{rashtchian2010collecting} with $1024$-dim for image and $100$-dim for text. It is a challenging visual dataset, where the text represents the context for each picture, but they are not as semantically rich as a full text article \cite{pereira2012regularization}. Following the same data split process on mfeat data,
we randomly select $50\%$ of the data for training and the rest for testing. Among the training data, $20\%$ data are randomly selected as paired  and the rest as unpaired. For semi-supervised learning, we randomly sample $10\%$ of training data as labeled and the rest as unlabeled. The same baseline methods for semi-supervised semi-paired learning in Section~\ref{sec:sssl-mfeat} are used in our comparison.

First in Table \ref{tab:pascal}, we show the average accuracy by the six methods with three different scatter matrices over $10$ randomly drawn training and testing splits. Similar conclusions  to what we had for mfeat can be drawn, namely, our proposed uncorrelated models outperform their counterparts. Note that USCCA produces dramatically much better results than SCCA. 
Moreover, we show the classification accuracy by each compared method on one of two views using the learned projection matrix instead of the concatenation of two views. The testing results on each single view show the similar conclusion. 
Second in Fig.~\ref{fig:Pascal}, we show the accuracy by all compared methods as the reduced dimension $k$ varies in $[2, 10]$. It is observed that our proposed method US$^2$GCA-mfa outperforms others over all $k$, and  USCCA-mfa shows the second best result similarly to S$^2$GCA-mfa. 
These experimental results demonstrate that our proposed models with uncorrelated constraints achieve better results than the baseline methods.

\section{Conclusion} \label{sec:conclusions}

We have proposed a generalized  semi-paired subspace learning framework to  jointly learn latent common space across two views and uncorrelated features. We demonstrate the flexibility of our proposed framework by showcasing
five novel models which are then compared with similar existing models. Extensive experiments show that the integration of semi-paired subspace learning with learning uncorrelated features can be beneficial for both unsupervised learning and semi-supervised learning. Moreover we design a successive alternating approximation (SAA) method to numerically solve the general framework. The
method can be directly used for solving any model that fits in this framework. The potential extension of this work to more than two views and nonlinear transformation via kernel trick will be investigated elsewhere.


\end{document}